\documentclass{article}
\PassOptionsToPackage{numbers, compress}{natbib}

\usepackage[preprint]{neurips_2026}

\usepackage[utf8]{inputenc} 
\usepackage[T1]{fontenc}    
\usepackage{hyperref}       
\usepackage{url}            
\usepackage{booktabs}       
\usepackage{amsfonts}       
\usepackage{nicefrac}       
\usepackage{microtype}      
\usepackage{xcolor}         
\usepackage{colortbl}
\usepackage{multirow}
\usepackage{diagbox}
\usepackage{subcaption}
\usepackage{makecell}
\usepackage{setspace}
\usepackage{wrapfig}
\usepackage{placeins}

\usepackage{amsmath}
\usepackage{amssymb}
\usepackage{mathtools}
\usepackage{amsthm}
\usepackage{pifont}

\usepackage{bm}
\usepackage{algorithm}
\usepackage{algpseudocode}

\newtheorem{theorem}{Theorem}
\algrenewcommand\algorithmicrequire{\textbf{Require:}}
\algrenewcommand\algorithmicensure{\textbf{Ensure:}}
\algrenewcommand\algorithmicreturn{\textbf{Return}}

\title{OMP-MoE: Efficient Expert Pruning for Mixture-of-Experts LLMs via Orthogonal Matching Pursuit}

\author{%
  Dezhi Li \quad Lujun Li \quad Qiyuan Zhu \quad Hao Gu \quad Bei Liu \quad Sirui Han\thanks{Corresponding authors.}
  \quad Yike Guo\footnotemark[1] \\
  The Hong Kong University of Science and Technology
}

\begin{document}

\maketitle

\begin{abstract}
  Mixture-of-Experts (MoE) models enable efficient scaling of large language models but face critical deployment challenges due to massive memory requirements. Existing pruning methods either incur prohibitive search costs or neglect the dynamic interdependencies between experts. To address these challenges, we present OMP-MoE, a novel training-free compression framework for reducing expert redundancy in MoE-based LLMs. Based on observations of expert contribution patterns, we reformulate the pruning problem as a sparse signal reconstruction task solved through Orthogonal Matching Pursuit. Specifically, our method first treats individual expert contributions as dictionary atoms and selects experts that greedily minimize reconstruction error with linear computational complexity. Then, we optimize cross-layer expert allocation through a water-filling strategy that accounts for both reconstruction quality and routing stability. Finally, we introduce OMP-MoE$^\dagger$, an adaptive inference mechanism that dynamically adjusts expert activation based on energy prediction. Comprehensive experiments on Qwen, DeepSeek-V2, GPT-OSS, and Mixtral MoE demonstrate consistent improvements over existing methods at 25-50\% pruning ratios. For Qwen3-30B-A3B at 50\% expert pruning, we retain 93.3\% of original performance, achieving 33$\times$ faster search and 1.55$\times$ inference speedup. Codes are available in Supplementary Material.
\end{abstract}

\begin{figure*}[th]
    \centering
    \begin{subfigure}[b]{0.33\linewidth}
        \centering
        \includegraphics[width=\linewidth]{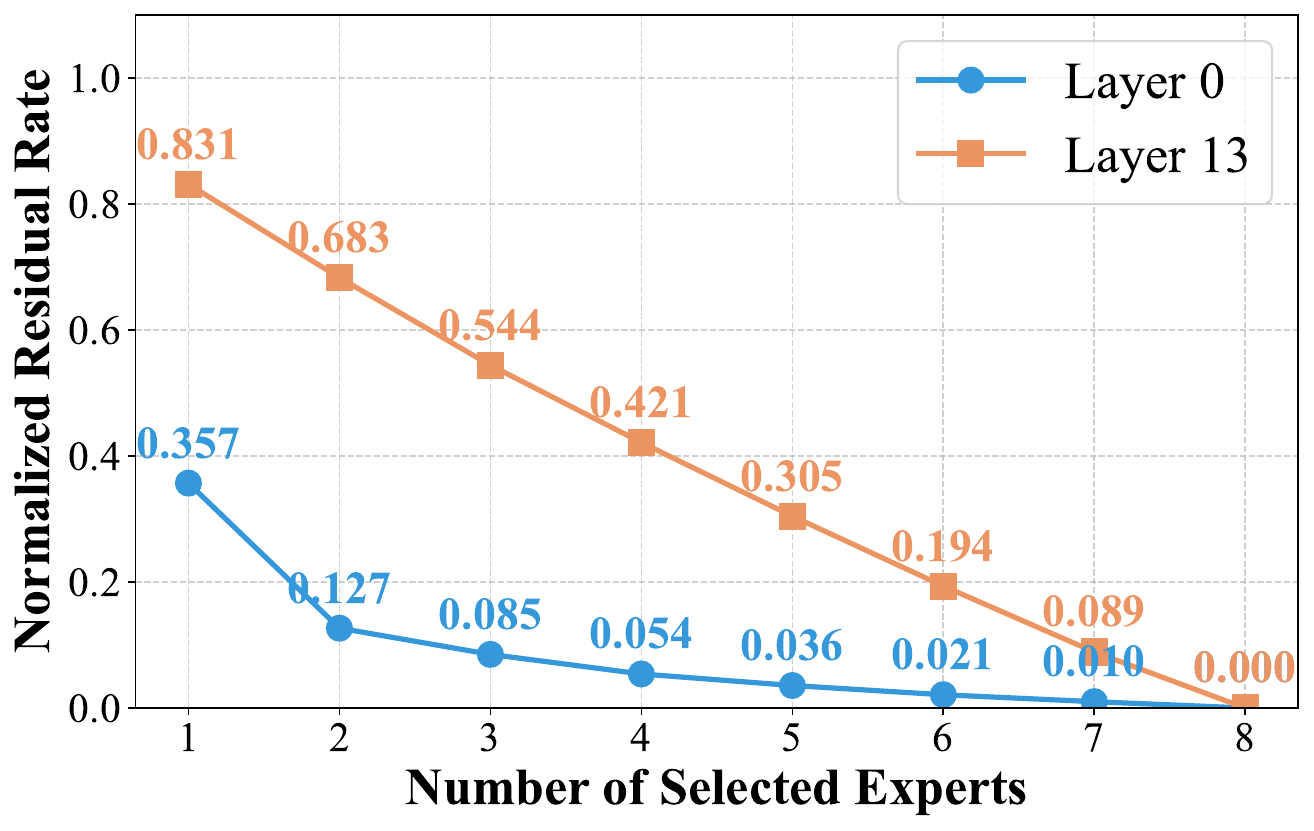}
        \caption{Normalized residual curves}
        \label{fig:analysis_residual}
    \end{subfigure}
    \hfill
    \begin{subfigure}[b]{0.63\linewidth}
        \centering
        \includegraphics[width=\linewidth]{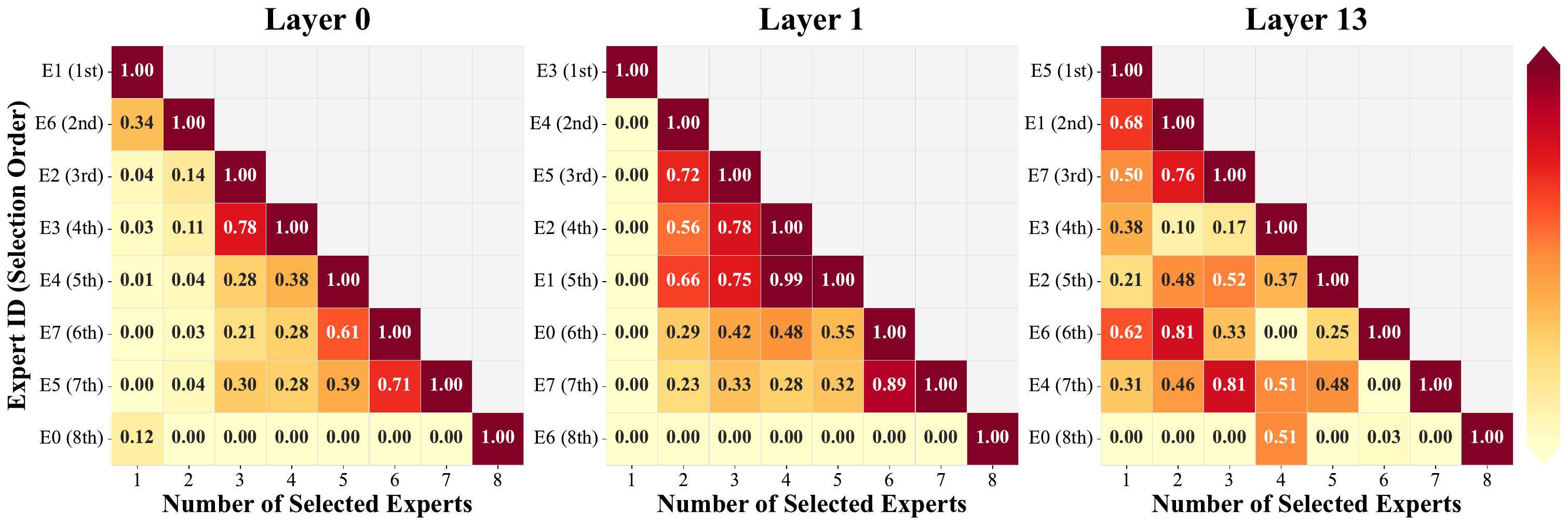}
        \caption{Reconstruction gain heatmaps generated by OMP}
        \label{fig:analysis_heatmap}
    \end{subfigure}

    \caption{\textbf{Analysis of Dynamic Expert Selection.}
    We visualize the selection process across three representative layers in Mixtral-8$\times$7B:
    (\textbf{a}) Decay of normalized residual rate curve as a function of the number of selected experts.
    (\textbf{b}) Dynamic importance of candidate experts sorted by selection order. We provide more details and visualizations across various MoE LLMs in Appendix~\ref{app:intro:heatmaps} and~\ref{app:intro:residual_curves}.}
    \label{fig:combined_analysis}
\end{figure*}

\section{Introduction}

The scaling of Large Language Models (LLMs) has led to unprecedented capabilities across diverse tasks~\cite{report2023gpt4, report2024llama3}. To balance massive parameter counts with computational efficiency, the Mixture-of-Experts (MoE) architecture has emerged as a dominant paradigm~\cite{shazeer2017outrageously, fedus2022switch}. Modern MoE models, such as Mixtral~\cite{jiang2024mixtral}, DeepSeek-V3~\cite{deepseekai2024deepseekv3}, Qwen-MoE~\cite{yang2025qwen3}, and GPT-OSS~\cite{agarwal2025gptoss}, utilize sparse gating to activate only a subset of experts per token, and build upon foundational works that introduced sparsity to scale capacity~\cite{shazeer2017outrageously, lepikhin2021gshard, fedus2022switch}. However, the immense scale of these models, which exceeds hundreds of billions of parameters, presents significant challenges for deployment in environments with limited resources~\cite{ghorbani2022scaling, Clark2022ScalingMoe}. 

Existing works on MoE compression primarily follow two trajectories: expert merging and expert pruning. Merge-based methods, such as MC-SMoE~\cite{li2023mcsmoe} and HC-SMoE~\cite{chen2024hcsmoe}, attempt to combine redundant experts into a reduced set of ``centroid'' experts. While Sub-MoE~\cite{li2025submoe} improves upon this by merging experts within a subspace, recent studies like REAP~\cite{lasby2025reap} argue that merging inevitably introduces lossy representations by blending distinct specialized features. In contrast, expert pruning directly removes less critical experts to reduce the model's footprint. However, current pruning techniques face a significant efficiency-performance trade-off. Combination-based methods like NAEE~\cite{lu2024naee} require exhaustive evaluations over combinatorial spaces. For a model with $N_e=128$ experts and a target of $n=64$, the search space reaches $\binom{128}{64} \approx 2.4 \times 10^{37}$ combinations, making global optimization prohibitive. Other approaches like HEAPr~\cite{li2025heapr} leverage Hessian-based second-order information to guide pruning, but calculating the Hessian inverse for massive models incurs prohibitive memory and computational overhead. Similarly, differentiable methods such as DiEP~\cite{bai2025diep} necessitate additional fine-tuning, which imposes a heavy training burden.
These challenges present the key question: \textit{How can we design efficient pruning frameworks that simultaneously achieve fast search and maintain performance?}

\textbf{\textit{``The shortest distance between two points is a straight line.''}}

\textbf{\rightline{\textit{--- Euclid}}}

As the quote suggests, efficient solutions often emerge from direct formulations that exploit underlying structure.
Motivated by sparse approximation, we analyze whether expert pruning can be posed as reconstructing the MoE layer output from a small subset of expert contribution atoms. Figure~\ref{fig:combined_analysis} shows two useful properties. First, residual curves decay at different rates across layers, indicating that a uniform expert budget is suboptimal. Second, expert importance changes after each selected expert because the residual signal changes, making static one-shot ranking insufficient. These findings underscore that employing Orthogonal Matching Pursuit (OMP)~\cite{pati1993orthogonal,tropp2007signal} for expert selection is a promising approach for MoE compression that balances efficiency, optimality, and performance preservation.


Building on these findings, we propose OMP-MoE, a training-free expert pruning framework for MoE LLMs.
OMP-MoE formulates expert pruning as a sparse signal reconstruction problem, where expert contributions are treated as dictionary atoms and selected by greedily reducing the current reconstruction residual.
This design avoids exhaustive subset search, reduces the search complexity from exponential to linear, and provides a locally optimal greedy selection rule under the binary expert-retention constraint.
To handle layer heterogeneity, OMP-MoE further allocates expert budgets across layers with a water-filling strategy that combines reconstruction error and routing stability.
We also introduce OMP-MoE$^{\dagger}$, an adaptive inference mechanism that skips low-energy expert executions for input-aware computation reduction.


Extensive experiments demonstrate that the OMP-MoE framework identifies superior expert subsets and consistently outperforms state-of-the-art methods across multiple architectures. Our experiments reveal three primary advantages. First, OMP-MoE achieves superior performance retention: at 50\% pruning on Qwen3-30B-A3B, it maintains 93.3\% of the original model's accuracy, substantially outperforming the 53.0\% retained by MC-SMoE and the 80.5\% retained by NAEE. Second, our method exhibits remarkable search efficiency, completing expert selection in just 641 seconds, a 33$\times$ speedup over NAEE while maintaining comparable memory usage. Third, our adaptive OMP-MoE$^\dagger$ delivers tangible inference benefits, achieving 1.46$\times$ end-to-end acceleration on DeepSeek-V2-Lite at 50\% pruning. Beyond standard benchmarks, OMP-MoE demonstrates strong capability retention on reasoning-intensive benchmarks, preserving 59.6\% performance on GSM8K compared to 26.5\% for the best baseline. Furthermore, OMP-MoE exhibits orthogonal compatibility with quantization and unstructured pruning methods.

\section{Related Work}

\textbf{Mixture-of-Experts Compression.}
The compression of MoE models has seen rapid development as parameter counts scale. Expert merging methods consolidate weights to reduce memory footprints. MC-SMoE~\cite{li2023mcsmoe} groups and merges experts based on routing patterns, though this fusion often degrades the fine-grained specialization of the original model. HC-SMoE~\cite{chen2024hcsmoe} utilizes hierarchical clustering to optimize similarity metrics, yet its static nature fails to account for dynamic expert interactions during inference. Sub-MoE~\cite{li2025submoe} employs subspace expert merging to preserve dominant components, but blending distinct expert functionalities remains fundamentally lossy. Expert pruning directly removes redundant parameters to improve efficiency. NAEE~\cite{lu2024naee} identifies critical experts through combinatorial subset evaluation, which becomes computationally prohibitive as the number of experts increases. DiEP~\cite{bai2025diep} utilizes differentiable masks for expert selection, necessitating a resource-intensive fine-tuning stage. Shapley-MoE~\cite{huang2025shapleymoe} applies game-theoretic frameworks to estimate marginal contributions, but the exact calculation is often intractable for massive MoE architectures. HEAPr~\cite{li2025heapr} guides pruning via Hessian-based second-order information, which incurs significant memory and computational overhead for high-dimensional models. In contrast to prior approaches hindered by prohibitive optimization costs or the neglect of inter-expert correlations, OMP-MoE employs a training-free, iterative mechanism to select experts with linear computational complexity.

\textbf{Orthogonal Matching Pursuit.}
Orthogonal Matching Pursuit is a classic greedy algorithm designed for sparse signal reconstruction in compressed sensing~\cite{tropp2007signal, pati1993orthogonal}. It iteratively selects dictionary atoms most correlated with the current residual and updates the residual via orthogonal projection. In neural network optimization, OMP has been used for filter pruning in CNNs~\cite{Yang2009CNN} and efficient sparse coding for KV caches~\cite{kim2025lexico}. We first explore OMP for MoE compression, treating individual experts as dictionary atoms to solve the pruning problem.

\begin{figure*}[t]
    \vspace{-0.8cm}
    \centering
    \includegraphics[width=\textwidth]{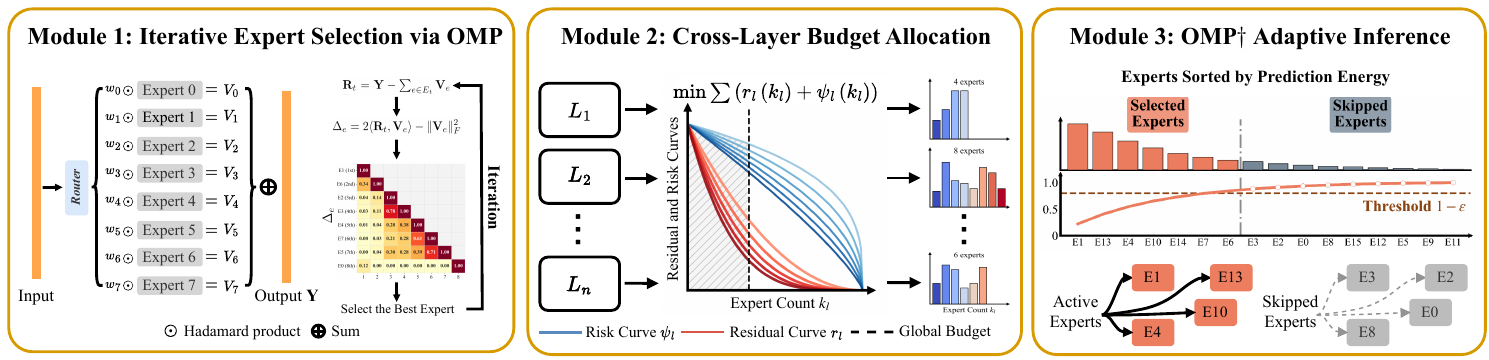}
    \caption{Overview of our OMP-MoE framework: (a) Greedy pursuit of expert contribution atoms through iterative selection to minimize layer wise reconstruction error. (b) Cross-layer resource optimization via water-filling to align expert counts with marginal benefits. (c) Energy-based dynamic pruning during runtime to achieve streamlined inference without performance degradation.}
    \label{fig:framework_overview}
\end{figure*}

\section{Methodology}
We propose OMP-MoE, a framework designed to efficiently identify the optimal expert subset under the iterative principles of Orthogonal Matching Pursuit. As illustrated in Figure~\ref{fig:framework_overview}, OMP-MoE is structured into three main components: (1) single-layer expert selection via OMP, (2) cross-layer budget allocation with routing stability, and (3) adaptive inference with OMP-MoE$^\dagger$. We provide the detailed theoretical derivation of the OMP-MoE framework and the formal proofs of optimality in Appendix~\ref{app:formal_analysis}.

\subsection{Preliminary}

\textbf{Mixture-of-Experts Mechanism.} Given an MoE layer consisting of $N_e$ experts $\{F_e\}_{e=1}^{N_e}$, the output $\mathbf{Y}$ for an input $\mathbf{X}$ is computed as the sparse weighted sum of expert outputs:
\begin{equation}
\mathbf{Y} = \sum_{e=1}^{N_e} G(\mathbf{X})_e \odot F_e(\mathbf{X}) = \sum_{e=1}^{N_e} \mathbf{V}_e,
\end{equation}
where $\odot$ denotes the broadcasting element-wise product and $\mathbf{V}_e$ represents the \textit{contribution atom} of expert $e$. The gating scores $G(\mathbf{X})$ are computed from the routing matrix $\mathbf{W}_g \in \mathbb{R}^{d \times N_e}$ by the architecture-native router:
\begin{equation}
G(\mathbf{X}) = \text{Router}(\mathbf{X}\mathbf{W}_g),
\label{eq:router}
\end{equation}
where $\text{Router}(\cdot)$ denotes the model's native composition of score transformation, sparse expert selection, and any normalization or bias correction. A common instance applies softmax scores followed by top-$k$ selection, as in Mixtral's top-2 router~\cite{jiang2024mixtral}. Other architectures such as DeepSeek-V3~\cite{deepseekai2024deepseekv3} use sigmoid scores and top-8 selection.

In implementation, $\mathbf{X}$ denotes a mini-batch of token representations at a given MoE layer. For $B$ sequences with length $T$, we flatten them into $N=B\times T$ tokens, so $\mathbf{Y}\in\mathbb{R}^{N\times d}$ and $\mathbf{V}_e\in\mathbb{R}^{N\times d}$ store token-wise layer outputs and weighted expert contributions.

\textbf{Principles of Orthogonal Matching Pursuit.} Orthogonal Matching Pursuit (OMP)~\cite{pati1993orthogonal} is a classic greedy algorithm designed for sparse approximation:
\begin{equation}
\min_{\mathbf{c}} \|\mathbf{Y} - \mathbf{D}\mathbf{c}\|_F^2 \quad \text{subject to} \quad \|\mathbf{c}\|_0 \le n,
\label{eq:omp_standard}
\end{equation}
where $\mathbf{D} = [\mathbf{d}_1, \dots, \mathbf{d}_{N_e}]$ is a dictionary of atoms. OMP iteratively selects the atom $\mathbf{d}_i$ that maximizes the projection onto the current residual signal and updates the coefficient vector $\mathbf{c}$ via orthogonal projection to minimize the reconstruction error.

We adopt the OMP-MoE framework by treating the set of contribution atoms $\{\mathbf{V}_1, \dots, \mathbf{V}_{N_e}\}$ as our dictionary. Crucially, we constrain the optimization coefficients to be binary, i.e., $c_e \in \{0, 1\}$, indicating whether an expert is retained. This yields the following reconstruction objective for expert selection:
\begin{equation}
\min_{E: |E|=n} \left\| \mathbf{Y} - \sum_{e \in E} \mathbf{V}_e \right\|_F^2,
\label{eq:reconstruction_obj}
\end{equation}

where $E$ represents the support set containing the indices of $n$ selected experts. By reformulating pruning as a signal reconstruction task, we transform a combinatorial search into an efficient greedy pursuit of the optimal expert subset.

\subsection{Iterative Expert Selection via OMP}

To solve the sparse approximation problem in Eq.~\ref{eq:reconstruction_obj}, we adopt the greedy selection strategy from OMP theory. Instead of performing an exhaustive combinatorial search over $C(N_e, n)$ directly, OMP iteratively selects the expert that provides the maximum reduction in the current residual rate.

\begin{algorithm}[h]
\caption{Greedy Expert Selection via OMP}
\begin{algorithmic}
\Require Layer output $\bm{Y}$, contribution atoms $\{\bm{V}_e\}_{e=1}^{N_e}$, expert norms $\{\|\bm{V}_e\|_F^2\}_{e=1}^{N_e}$, target count $n$
\Ensure Set of retained experts $E^{\mathrm{keep}}$

\State Initialize residual $\bm{R}_0 \gets \bm{Y}$
\State Initialize selected set $E_0 \gets \varnothing$

\For{$t = 0$ to $n-1$}
    \State $\Delta_e \gets 2\langle \bm{R}_t, \bm{V}_e\rangle
    - \|\bm{V}_e\|_F^2$ for all $e \notin E_t$
    \State $e^\ast \gets \operatorname*{arg\,max}_{e \notin E_t} \Delta_e$
    \State $\bm{R}_{t+1} \gets \bm{R}_t - \bm{V}_{e^\ast}$
    \State $E_{t+1} \gets E_t \cup \{e^\ast\}$
\EndFor

\State \Return $E_n$
\end{algorithmic}
\label{alg:greedy-expert-omp}
\end{algorithm}

\textbf{Analytical Selection Criterion.} Let $E_t$ be the set of experts selected after $t$ rounds, and $\mathbf{R}_t = \mathbf{Y} - \sum_{e \in E_t} \mathbf{V}_e$ be the current residual signal (with $\mathbf{R}_0 = \mathbf{Y}$). When evaluating a candidate expert $e \notin E_t$, the reconstruction gain $\Delta_e$ can be derived analytically in a single step:
\begin{equation}
\Delta_e = 2\langle \mathbf{R}_t, \mathbf{V}_e \rangle - \|\mathbf{V}_e\|_F^2
\label{eq:omp_gain_main}
\end{equation}
where $\langle \cdot, \cdot \rangle$ denotes the Frobenius inner product. We provide the detailed mathematical derivation of this formula in Appendix~\ref{app:formal_analysis}.
This criterion is locally optimal under the binary expert-retention constraint because, conditioned on the current selected set, maximizing $\Delta_e$ is exactly equivalent to minimizing the next residual norm. The full proof is given in Appendix~\ref{app:binary_greedy_optimality}.

To identify the critical experts, we employ an iterative greedy search. This process consists of two phases: a one-time \textit{Statistics Phase} that performs a single forward pass to cache the signal $\mathbf{Y}$ and expert contribution atoms $\{\mathbf{V}_e\}_{e=1}^{N_e}$ for each layer, and a \textit{Greedy Selection Phase} that iteratively updates the residual and selects experts through tensor operations. As detailed in Algorithm~\ref{alg:greedy-expert-omp}, this iterative approach generates an ordered contribution sequence of experts.

\textbf{Search Efficiency.} As derived in Appendix~\ref{app:sub:complexity_analysis}, our greedy selection strategy achieves a total time complexity of $\mathcal{O}(L \cdot n \cdot N_e \cdot Bd)$, where $L$ is the number of layers, $n$ is the target number of experts, $N_e$ is the total number of experts, and $Bd$ represents the token-dimension product of the cached statistics. This provides a significant computational speedup over the evaluation of all $C(N_e, n)$ expert combinations, which is prohibitive for modern MoE models with large $N_e$. By decoupling the selection logic from the forward pass and operating on a linear search space, OMP-MoE can prune models with over 128 experts within minutes.

\subsection{Cross-Layer Budget Allocation with Routing Stability}
\label{subsec:cross_layer_allocation}

Uniform pruning assigns the same expert budget to each MoE layer, although different layers may have different reconstruction difficulty. Moreover, removing experts changes the normalization term of the gating softmax, leading to routing renormalization~\cite{dai2022stablemoe}. We therefore allocate experts across layers by jointly considering reconstruction error and preserved routing mass.

For layer $l$, OMP-MoE produces an ordered expert sequence
$\pi_l=(e_{l,1},e_{l,2},\dots,e_{l,N_e})$. If the first $k$ experts are retained, the residual is
\begin{equation}
\mathbf{R}_l(k)=\mathbf{Y}_l-\sum_{t=1}^{k}\mathbf{V}_{l,e_{l,t}},
\end{equation}
and the normalized reconstruction error is
\begin{equation}
r_l(k)=\frac{\|\mathbf{R}_l(k)\|_F^2}{\|\mathbf{Y}_l\|_F^2}.
\end{equation}

To measure the routing mass preserved by the retained experts, we accumulate the gating weight
$m_{l,e}$ of expert $e$ on the calibration set and normalize it as
\begin{equation}
\tilde{m}_{l,e}=\frac{m_{l,e}}{\sum_{j=1}^{N_e}m_{l,j}}.
\end{equation}
The routing coverage of the first $k$ retained experts is
\begin{equation}
C_l(k)=\sum_{t=1}^{k}\tilde{m}_{l,e_{l,t}},
\end{equation}
and we define the routing risk as
\begin{equation}
\psi_l(k)=-\log(C_l(k)+\zeta),
\end{equation}
where $\zeta$ is a small constant for numerical stability. The layer-wise allocation cost is then
\begin{equation}
F_l(k)=r_l(k)+\lambda\psi_l(k),
\end{equation}
where $\lambda$ controls the strength of the routing-risk term.

Given a total expert budget $C_{\mathrm{tot}}$, we choose the number of retained experts
$\{k_l\}_{l=1}^{L}$ by solving
\begin{equation}
\min_{\{k_l\}_{l=1}^{L}}
\sum_{l=1}^{L}F_l(k_l)
\quad
\mathrm{s.t.}\quad
\sum_{l=1}^{L}k_l=C_{\mathrm{tot}},\quad
k_l\in\{1,\dots,N_e\}.
\label{eq:cross_layer}
\end{equation}

We solve Eq.~\eqref{eq:cross_layer} with a discrete water-filling procedure. We initialize each layer with one retained expert and then allocate the remaining $C_{\mathrm{tot}}-L$ experts one by one. At each step, we add one expert to the layer with the largest marginal cost decrease:
\begin{equation}
l^\star=\arg\max_{l:k_l<N_e}
\Delta_l(k_l+1),
\quad
\Delta_l(k_l+1)=F_l(k_l)-F_l(k_l+1).
\end{equation}
The update is then $k_{l^\star}\leftarrow k_{l^\star}+1$. This greedy allocation places each added expert in the layer where it gives the largest immediate decrease in the sum of reconstruction error and routing risk. More details are given in Appendix~\ref{app:sub:detailed_workflow}.

\begin{table*}[t]
    \centering
    \setlength{\extrarowheight}{0pt}
    \addtolength{\extrarowheight}{\aboverulesep}
    \addtolength{\extrarowheight}{\belowrulesep}
    \setlength{\aboverulesep}{0pt}
    \setlength{\belowrulesep}{0pt}
    \caption{Zero-shot performance comparison across 6 tasks on Qwen, DeepSeek, GPT-OSS and Mixtral MoE models at 25\% and 50\% expert pruning ratios. Bold indicates the best performance. `Avg.' represents the average accuracy across 6 tasks.}
    \label{tab:main_results}
    \resizebox{\textwidth}{!}{
    \begin{tabular}{lccccccc||ccccccc}
    \toprule
    \textbf{Method} & \textbf{ARC-c} & \textbf{BoolQ} & \textbf{HellaS.} & \textbf{MMLU} & \textbf{OBQA} & \textbf{WinoG.} & \textbf{Avg.} & \textbf{ARC-c} & \textbf{BoolQ} & \textbf{HellaS.} & \textbf{MMLU} & \textbf{OBQA} & \textbf{WinoG.} & \textbf{Avg.} \\
    \midrule
    \rowcolor[rgb]{0.851,0.933,0.98}  & \multicolumn{7}{c||}{\textbf{Qwen3-30B-A3B}} & \multicolumn{7}{c}{\textbf{DeepSeek-V2-Lite }} \\
    Original & 0.528 & 0.887 & 0.596 & 0.778 & 0.346 & 0.703 & 0.640 & 0.465 & 0.799 & 0.587 & 0.551 & 0.348 & 0.710 & 0.577 \\
    \midrule
     & \multicolumn{7}{c||}{Pruning Ratio 25\%} & \multicolumn{7}{c}{Pruning Ratio 25\%} \\
    \cmidrule(l){2-15}
    MC-SMoE & 0.392 & 0.777 & 0.415 & 0.540 & 0.318 & 0.588 & 0.505 & 0.367 & 0.713 & 0.531 & 0.422 & \textbf{0.366} & 0.687 & 0.514 \\
    HC-SMoE & 0.458 & 0.865 & 0.515 & 0.669 & 0.410 & \textbf{0.704} & 0.603 & 0.420 & 0.722 & 0.560 & 0.458 & 0.280 & 0.695 & 0.523 \\
    NAEE & 0.481 & 0.870 & 0.555 & 0.701 & 0.290 & 0.693 & 0.598 & 0.375 & 0.669 & 0.531 & 0.365 & 0.290 & 0.669 & 0.483 \\
    DiEP & 0.505 & 0.871 & 0.567 & 0.646 & 0.334 & 0.702 & 0.604 & 0.433 & 0.740 & 0.550 & \textbf{0.505} & 0.302 & 0.660 & 0.532 \\
    MoNE & 0.523 & 0.871 & 0.580 & 0.730 & 0.308 & 0.684 & 0.616 & 0.456 & 0.720 & 0.576 & 0.474 & 0.314 & 0.702 & 0.540 \\
    Shapley-MoE & 0.458 & 0.841 & 0.524 & - & 0.300 & 0.650 & - & 0.395 & \textbf{0.760} & 0.512 & - & 0.262 & 0.673 & - \\
    \rowcolor[rgb]{0.992,0.992,0.741} OMP-MoE (Ours) & \textbf{0.534} & \textbf{0.890} & \textbf{0.592} & \textbf{0.747} & \textbf{0.338} & 0.699 & \textbf{0.633} & \textbf{0.462} & 0.703 & \textbf{0.583} & 0.484 & 0.324 & \textbf{0.709} & \textbf{0.544} \\
    \midrule
     & \multicolumn{7}{c||}{Pruning Ratio 50\%} & \multicolumn{7}{c}{Pruning Ratio 50\%} \\
    \cmidrule(l){2-15}
    MC-SMoE & 0.195 & 0.583 & 0.273 & 0.261 & 0.250 & 0.470 & 0.339 & 0.288 & 0.595 & 0.438 & 0.237 & 0.306 & 0.598 & 0.410 \\
    HC-SMoE & 0.358 & 0.830 & 0.418 & 0.461 & \textbf{0.362} & 0.652 & 0.513 & 0.325 & 0.564 & 0.456 & \textbf{0.302} & 0.216 & 0.625 & 0.415 \\
    NAEE & 0.369 & 0.805 & 0.467 & 0.521 & 0.262 & 0.665 & 0.515 & 0.328 & 0.583 & 0.409 & 0.265 & 0.218 & 0.585 & 0.398 \\
    DiEP & 0.389 & 0.762 & 0.486 & 0.299 & 0.244 & 0.634 & 0.469 & 0.294 & 0.617 & 0.441 & 0.256 & 0.212 & 0.569 & 0.398 \\
    MoNE & 0.387 & 0.866 & \textbf{0.553} & 0.516 & 0.286 & 0.673 & 0.547 & \textbf{0.375} & 0.548 & 0.484 & 0.253 & 0.258 & 0.611 & 0.422 \\
    Shapley-MoE & 0.320 & 0.712 & 0.420 & - & 0.222 & 0.591 & - & 0.268 & 0.240 & 0.395 & - & 0.240 & 0.603 & - \\
    \rowcolor[rgb]{0.992,0.992,0.741} OMP-MoE (Ours) & \textbf{0.533} & \textbf{0.879} & 0.546 & \textbf{0.605} & 0.324 & \textbf{0.695} & \textbf{0.597} & 0.354~ & \textbf{0.602~} & \textbf{0.513~} & 0.273~ & \textbf{0.310~} & \textbf{0.631~} & \textbf{0.447} \\
    \midrule
    \rowcolor[rgb]{0.851,0.933,0.98}  & \multicolumn{7}{c||}{\textbf{GPT-OSS-20B}} & \multicolumn{7}{c}{\textbf{Mixtral-8$\times$7B}} \\
    Original & 0.451 & 0.757 & 0.415 & 0.566 & 0.27 & 0.657 & 0.519 & 0.565 & 0.851 & 0.649 & 0.671 & 0.350 & 0.759 & 0.641 \\
    \midrule
     & \multicolumn{7}{c||}{Pruning Ratio 25\%} & \multicolumn{7}{c}{Pruning Ratio 25\%} \\
    \cmidrule(l){2-15}
    MC-SMoE & 0.399 & 0.725 & 0.399 & 0.488 & 0.258 & 0.665 & 0.489 & 0.262 & 0.521 & 0.432 & 0.250 & 0.194 & 0.585 & 0.374 \\
    HC-SMoE & 0.285 & 0.622 & 0.336 & 0.425 & 0.182 & 0.617 & 0.411 & 0.450 & 0.830 & 0.570 & 0.560 & 0.290 & 0.745 & 0.574 \\
    NAEE & 0.422 & 0.731 & 0.400 & 0.547 & 0.232 & 0.624 & 0.493 & 0.516 & \textbf{0.836} & 0.616 & 0.587 & 0.330 & 0.754 & 0.607 \\
    DiEP & 0.379 & 0.751 & 0.393 & 0.527 & 0.224 & 0.593 & 0.478 & 0.514 & 0.835 & 0.612 & 0.598 & 0.302 & 0.740 & 0.600 \\
    MoNE & 0.421 & 0.653 & 0.406 & 0.503 & \textbf{0.268} & \textbf{0.654} & 0.484 & 0.463 & 0.702 & 0.594 & 0.493 & 0.300 & 0.728 & 0.547 \\
    \rowcolor[rgb]{0.992,0.992,0.741} OMP-MoE (Ours) & \textbf{0.424} & \textbf{0.762} & \textbf{0.409} & \textbf{0.555} & 0.260 & 0.650 & \textbf{0.510} & \textbf{0.535} & 0.831 & \textbf{0.624} & \textbf{0.604} & \textbf{0.330} & \textbf{0.766} & \textbf{0.615} \\
    \midrule
     & \multicolumn{7}{c||}{Pruning Ratio 50\%} & \multicolumn{7}{c}{Pruning Ratio 50\%} \\
    \cmidrule(l){2-15}
    MC-SMoE & \textbf{0.337} & 0.716 & 0.365 & 0.376 & 0.236 & 0.635 & 0.444 & 0.212 & 0.495 & 0.277 & 0.245 & 0.108 & 0.496 & 0.306 \\
    HC-SMoE & 0.208 & 0.421 & 0.31 & 0.295 & 0.174 & 0.555 & 0.327 & 0.322 & 0.754 & 0.493 & 0.392 & 0.256 & 0.671 & 0.481 \\
    NAEE & 0.284 & 0.659 & 0.359 & 0.367 & 0.178 & 0.580 & 0.405 & \textbf{0.489} & 0.814 & 0.577 & 0.473 & 0.290 & 0.729 & 0.562 \\
    DiEP & 0.277 & 0.49 & 0.336 & 0.315 & 0.194 & 0.549 & 0.360 & 0.473 & 0.812 & \textbf{0.580} & 0.489 & 0.292 & \textbf{0.740} & 0.564 \\
    MoNE & 0.265 & 0.476 & 0.369 & 0.240 & 0.222 & 0.627 & 0.367 & 0.310 & 0.620 & 0.412 & 0.231 & 0.202 & 0.647 & 0.404 \\
    \rowcolor[rgb]{0.992,0.992,0.741} OMP-MoE (Ours) & 0.333 & \textbf{0.750} & \textbf{0.376} & \textbf{0.438} & \textbf{0.256} & \textbf{0.642} & \textbf{0.466} & 0.482 & \textbf{0.830} & 0.577 & \textbf{0.497} & \textbf{0.298} & 0.728 & \textbf{0.569} \\
    \bottomrule
    \end{tabular}
    }
\vspace{-10pt}
\end{table*}
\subsection{OMP-MoE$^{\dagger}$ : Adaptive Inference via Energy Prediction}

To alleviate the compute bottleneck during deployment, we introduce OMP-MoE$^\dagger$  as an adaptive inference strategy. This algorithm leverages the precomputed energy coefficients to skip unnecessary experts for each input token, thereby lowering the total computational overhead without sacrificing the fidelity of the reconstructed signal.

We precompute an energy coefficient $\alpha_e = \mathbb{E}[ \| F_e(\mathbf{x}) \|^2 ]$ which is set to quantify the potential reconstruction contribution of each expert in an one-time offline calibration phase. During inference, for a token $\mathbf{x}_i$ with gating weights $g_{i,j}$, we predict the contribution energy as $\hat{\mathcal{E}}_{i,j} = g_{i,j}^2 \cdot \alpha_{e_{i,j}}$. We then select the minimum number of experts $k_i$ required to satisfy a relative residual threshold $\epsilon$:
\begin{equation}
k_i = \min \{ m \in [1, K] \mid (1 - \frac{\sum_{j=1}^m \hat{\mathcal{E}}_{i,j}}{\sum_{j=1}^K \hat{\mathcal{E}}_{i,j} + \zeta}) \le \epsilon \}.
\end{equation}
where $\zeta$ is a small constant to ensure numerical stability.
Appendix~\ref{app:sensitivity_energy} further gives a gradient-sensitivity view showing that the energy rule controls the expected downstream loss perturbation by skipping low-energy expert directions.
Under this dynamic scheme, we execute only the first $k_i$ experts and treat the contributions of the remaining candidates as zero. As analysed in Appendix~\ref{app:sub:flops_derivation}, this mechanism yields a theoretical efficiency gain of $\eta = 1 - \mathbb{E}[m^* | \epsilon]/K$, where $\mathbb{E}[m^* | \epsilon]$ is the expected number of activated experts under threshold $\epsilon$. Consequently, OMP-MoE$^\dagger$ transforms the rigid execution path of standard MoE into a input-aware mechanism, significantly reducing floating-point operations while preserving the reasoning fidelity of the model.

\begin{table*}[t]
    \centering
    \small
    \caption{Comparison of search efficiency and peak memory usage of NAEE, DiEP and OMP-MoE at a 50\% pruning ratio. Measurements are taken on 8 $\times$ NVIDIA H20 GPUs.}
    \label{tab:search_efficiency}
    \resizebox{\textwidth}{!}{
    \begin{tabular}{lcccccccc}
    \toprule
    \multirow{2}{*}{Method} & \multicolumn{2}{c}{Qwen3-30B-A3B} & \multicolumn{2}{c}{DeepSeek-V2-Lite} & \multicolumn{2}{c}{GPT-OSS-20B} & \multicolumn{2}{c}{Mixtral-8$\times$7B} \\
     & Time (s) & Mem (GB) & Time (s) & Mem (GB) & Time (s) & Mem (GB) & Time (s) & Mem (GB) \\
    \midrule
    NAEE & 21301 & 79.4 & 8017 & 42.5 & 12787 & 66.3 & 7168 & 102.4 \\
    DiEP & 8972 & 267.0 & 2745 & 98.8 & 4189 & 122.3 & 5957 & 242.3 \\
    \textbf{OMP-MoE (Ours)} & \textbf{641} & \textbf{75.0} & \textbf{358} & \textbf{40.6} & \textbf{275} & \textbf{56.8} & \textbf{188} & \textbf{97.5} \\
    \bottomrule
    \end{tabular}
    }
\vspace{-10pt}
\end{table*}

\begin{table}[t]
    \centering
    \small
    \caption{Inference performance and efficiency trade offs on DeepSeek-V2-Lite model. We report the average zero-shot accuracy across 8 benchmarks with the total inference time, and the speedup at 25\% and 50\% pruning ratios. The analysis compares the static pruning of OMP-MoE against the OMP-MoE$^\dagger$ algorithm with cross layer allocation enabled.}
    \label{tab:inference_cost}
    \begin{tabular}{l|ccc|lc}
    \toprule
    Prun. Ratio & OMP-MoE & OMP-MoE$^\dagger$  & Avg. Acc & Cost $\downarrow$ & Speedup $\uparrow$ \\
    \midrule
    0 & - & - & 0.609 & 2459s & 1.00$\times$ \\
    \midrule
    25\% & \checkmark & - & 0.576 & 2222s & 1.11$\times$ \\
    25\% & \checkmark & \checkmark & 0.581 & 2019s & 1.22$\times$ \\
    \midrule
    50\% & \checkmark & - & 0.489 & 1825s & 1.35$\times$ \\
    50\% & \checkmark & \checkmark & 0.476 & 1690s & 1.46$\times$ \\
    \bottomrule
    \end{tabular}
\vspace{-10pt}
\end{table}

\section{Experiments}
We evaluate OMP-MoE across multiple MoE LLMs and conduct detailed ablation studies on its key design choices. Further implementation details, expanded results, and analyses are  available in Appendix~\ref{app:config},~\ref{app:detailed_main_reuslts}, and \ref{app:sensitivity_study}, respectively.


\subsection{Experimental Setup}

We conduct comprehensive evaluations across a diverse range of MoE architectures, including Qwen3-30B-A3B~\cite{yang2025qwen3}, DeepSeek-V2-Lite~\cite{liu2024deepseekv2}, GPT-OSS-20B~\cite{agarwal2025gptoss}, and Mixtral-8$\times$7B~\cite{jiang2024mixtral}. Our OMP-MoE is benchmarked against state-of-the-art merging (MC-SMoE~\cite{li2023mcsmoe}, HC-SMoE~\cite{chen2024hcsmoe}) and pruning baselines (NAEE~\cite{lu2024naee}, MoNE~\cite{zhang2025mone}, DiEP~\cite{bai2025diep}, Shapley-MoE~\cite{huang2025shapleymoe}). To assess linguistic and reasoning proficiency, we evaluate zero-shot accuracy on eight general benchmarks: ARC-Challenge and ARC-Easy~\cite{clark2018arc}, BoolQ~\cite{clark2019boolq}, HellaSwag~\cite{zellers2019hellaswag}, MMLU~\cite{hendrycks2021mmlu}, OpenBookQA~\cite{mihaylov2018openbookqa}, RTE~\cite{Bentivogli2009rte}, and WinoGrande~\cite{sakaguchi2020winogrande}. We further evaluate perplexity, compositional reasoning on BBH~\cite{suzgun2022bbh}, mathematical reasoning on GSM8K~\cite{cobbe2021gsm8k}, and long-context modeling on LongBench~\cite{bai2024longbench}. Implementation utilizes 64 WikiText-2~\cite{merity2017wikitext2} calibration samples with risk penalty coefficient $\lambda = 3.0$ and dynamic threshold $\epsilon = 0.1$ on $8 \times$ NVIDIA H20 GPUs.

\subsection{Main Results}

\textbf{General Capabilities and Comparative Analysis.} We evaluate the zero-shot performance of OMP-MoE across 6 tasks on Qwen, DeepSeek, GPT-OSS and Mixtral MoE model. Table~\ref{tab:main_results} summarizes the average accuracy on these models at 25\% and 50\% expert pruning ratios. OMP-MoE consistently achieves superior accuracy compared to state-of-the-art merging and pruning baselines. At a 50\% expert pruning ratio, OMP-MoE maintains an average accuracy of 0.597 for the Qwen3-30B-A3B model and preserves 93.3\% of the original 0.640 baseline performance. Simultaneously, we observe similar patterns of resilience for other MoE models. These results confirm that our OMP-MoE effectively identifies the experts that are essential for model performance. We provide complete evaluation results across 8-task benchmarks in Appendix~\ref{app:detailed_main_reuslts}. The robustness and scaling results in Appendix~\ref{app:robustness_scaling} further show that OMP-MoE remains effective under more aggressive compression and scales to larger models.

\textbf{Pruning Overhead and Inference Acceleration.}
Table~\ref{tab:search_efficiency} reports the pruning overhead at a 50\% pruning ratio. On Qwen3-30B-A3B, OMP-MoE finishes search in 641 seconds, achieving a 33$\times$ speedup over NAEE and a 14$\times$ speedup over DiEP, while using lower peak memory than both baselines. Similar gains are observed across DeepSeek-V2-Lite, GPT-OSS-20B, and Mixtral-8$\times$7B, confirming the efficiency of the linear OMP search. Table~\ref{tab:inference_cost} further shows that OMP-MoE reduces inference cost on DeepSeek-V2-Lite, with OMP-MoE$^\dagger$ reaching a 1.46$\times$ speedup at a 50\% pruning ratio by skipping low-energy expert executions.

\begin{table*}[t]
\centering
\small
\setlength{\tabcolsep}{3.5pt}

\begin{minipage}[t][0.22\textheight][t]{0.43\textwidth}
\centering
\caption{Evaluation of mathematical reasoning, long-context modeling, and compositional reasoning tasks on Qwen3-30B-A3B at a 50\% pruning ratio.}
\label{tab:complex_tasks}
\vfill
\resizebox{\columnwidth}{!}{
\begin{tabular}{lccc}
\toprule
Method & GSM8K & LongBench & BBH \\
\midrule
Original & 0.891 & 0.578 & 0.363 \\
MC-SMoE & 0.005 & 0.504 & 0.219 \\
HC-SMoE & 0.007 & 0.284 & 0.235 \\
NAEE & 0.236 & 0.470 & 0.310 \\
MoNE & 0.210 & 0.496 & 0.262 \\
DiEP & 0.007 & 0.329 & 0.299 \\
\textbf{OMP-MoE} & \textbf{0.531} & \textbf{0.523} & \textbf{0.323} \\
\bottomrule
\end{tabular}
}
\end{minipage}
\hfill
\begin{minipage}[t][0.22\textheight][t]{0.53\textwidth}
\centering
\caption{Evaluation of OMP-MoE combined with post-training quantization and pruning methods on DeepSeek-V2-Lite. We report the perplexity on WikiText-2 and C4 datasets and the average accuracy across 8 zero-shot tasks.}
\label{tab:quantization}
\vfill
\resizebox{\columnwidth}{!}{
\begin{tabular}{lccc}
\toprule
Method & WikiText-2 & C4 & Avg. Acc \\
\midrule
OMP-MoE 25\% & 13.824 & 10.618 & 0.576 \\
+ GPTQ & 14.236 & 12.438 & 0.541 \\
+ SparseGPT (4:8) & 15.673 & 13.709 & 0.525 \\
+ Wanda & 14.430 & 13.321 & 0.534 \\
\bottomrule
\end{tabular}
}
\end{minipage}

\vspace{-10pt}
\end{table*}

\textbf{Performance on Reasoning-Heavy and Long-Context Tasks.} We evaluate OMP-MoE on GSM8K~\cite{cobbe2021gsm8k}, LongBench~\cite{bai2024longbench}, and BBH~\cite{suzgun2022bbh} to test mathematical reasoning, long-context modeling, and compositional reasoning, which remain central in recent work on efficient reasoning~\cite{chen2026adaptive,zhu2026outlier} and long-context modeling~\cite{hu2026hierarchical}.
Table~\ref{tab:complex_tasks} presents the experimental results for the Qwen3-30B-A3B model. The model retains substantial reasoning ability even after removing 50\% of the experts, achieving 0.531 on GSM8K, 0.523 on LongBench, and 0.323 on BBH. Compared with NAEE, OMP-MoE improves GSM8K from 0.236 to 0.531, showing that the selected experts are important for difficult reasoning and long-context inputs.
Appendix~\ref{app:recent_pruning_baselines} adds direct comparisons with REAP~\cite{lasby2025reap} and EASY-EP~\cite{Dong2025easyep}, while Appendix~\ref{app:domain_specific_calibration} reports GSM8K and HumanEval~\cite{chen2021codex} under task-relevant calibration.

\textbf{Orthogonality with Quantization and Pruning.} We further combine 25\% expert pruning with GPTQ~\cite{frantar2023gptq}, SparseGPT~\cite{frantar2023sparsegpt}, and Wanda~\cite{sun2024wanda}.
Table~\ref{tab:quantization} shows that OMP-MoE remains compatible with quantization and unstructured pruning techniques on DeepSeek-V2-Lite.
After adding GPTQ, SparseGPT, or Wanda, the model still retains average accuracies of 0.541, 0.525, and 0.534, respectively, indicating that expert-level pruning can be combined with quantization and unstructured pruning for stronger compression.

\subsection{Ablation Studies}

\begin{wrapfigure}{r}{0.55\textwidth}
\captionsetup{type=table}
\vspace{-1.2cm}
    \centering

\begin{minipage}{0.55\textwidth}
    \centering
\begin{spacing}{0.92}
\caption{Performance of different cross-layer allocation methods on 8 tasks zero-shot average accuracy.}
\label{tab:global_allocation}
\resizebox{0.85\columnwidth}{!}{
\begin{tabular}{lcccc}
\toprule
\multirow{2}{*}{Method} & \multicolumn{2}{c}{Qwen3-30B-A3B} & \multicolumn{2}{c}{DeepSeek-V2-Lite} \\
\cmidrule(l){2-5}
 & 25\% & 50\% & 25\% & 50\% \\
\midrule
Sub-MoE & 0.601 & 0.574 & - & - \\
HEAPr & 0.590 & 0.480 & - & - \\
HC-SMoE & 0.644 & 0.536 & 0.552 & 0.456 \\
DiEP & 0.640 & 0.511 & 0.566 & 0.435 \\
\midrule
\multicolumn{5}{l}{\textbf{OMP-MoE}} \\
\multicolumn{1}{r}{ \textbf{\textit{w/o}} cross-layer} & 0.671 & 0.618 & 0.572 & 0.468 \\
\multicolumn{1}{r}{\textbf{\textit{w}} cross-layer} & \textbf{0.676} & \textbf{0.643} & \textbf{0.576} & \textbf{0.489} \\
\bottomrule
\end{tabular}
}
\end{spacing}
\end{minipage}
\vspace{-0.6cm}
\end{wrapfigure}

\textbf{Effectiveness of Cross-Layer Budget Allocation.}
Table~\ref{tab:global_allocation} compares allocation methods. At 50\% pruning, OMP-MoE reaches 0.643 on Qwen3-30B-A3B, exceeding the 0.511 achieved by DiEP and the 0.480 achieved by HEAPr. Relative to uniform allocation, it improves Qwen3 from 0.618 to 0.643 and DeepSeek-V2-Lite from 0.468 to 0.489, confirming effective layer-wise budgeting. The consistent gains on both architectures indicate that layer sensitivity is heterogeneous and cannot be captured by assigning the same pruning ratio to every layer.

\textbf{Impact of Renormalization Risk Penalty.} Table~\ref{tab:risk_penalty_accuracy} studies the routing-risk weight $\lambda$ at 50\% pruning. Reconstruction error alone ($\lambda=0$) yields 0.618 on Qwen3-30B-A3B, whereas including preserved routing mass raises accuracy to 0.643. Thus, $\psi_l(k)=-\log(C_l(k)+\zeta)$ helps avoid allocations that induce strong routing renormalization. The monotonic improvement as $\lambda$ increases further suggests that reconstruction fidelity and routing stability provide complementary signals for allocation.

\begin{wrapfigure}{r}{0.45\textwidth}
    \captionsetup{type=table}
    \vspace{-0.47cm}
        \centering
\begin{minipage}{0.45\textwidth}
    \centering
    \begin{spacing}{0.9}
    \caption{Performance of different routing-risk weights on Qwen3-30B-A3B at a 50\% pruning ratio.}
    \label{tab:risk_penalty_accuracy}
\resizebox{1.0\textwidth}{!}{
    \begin{tabular}{lcccc}
    \toprule
    \multirow{2}{*}{\textbf{Risk Function} $\psi_l(k)$} & \multicolumn{4}{c}{\textbf{Routing-Risk Weight} $\lambda$} \\
    \cmidrule(l){2-5}
     & 0 & 1 & 2 & 3 \\
    \midrule
    $-\log(C_l(k)+\zeta)$ & 0.618 & 0.622 & 0.625 & \textbf{0.643} \\
    \bottomrule
    \end{tabular}
}
\end{spacing}
\end{minipage}
\vspace{-0.8cm}
\end{wrapfigure}

\textbf{Impact of OMP-MoE$^\dagger$ Dynamic Threshold.}
The threshold $\epsilon$ controls the residual expert energy skipped at inference. In Figure~\ref{fig:skip_and_calib}(a), $\epsilon=0.1$ attains the best accuracy of 0.581 while skipping about 35\% of expert executions on DeepSeek-V2-Lite. Larger thresholds remove too much computation and reduce accuracy. This operating point therefore provides a favorable balance between predictive quality and inference efficiency without additional training.
\begin{figure}[t]
    \centering
    \begin{subfigure}[t]{0.48\textwidth}
        \centering
        \includegraphics[width=\linewidth]{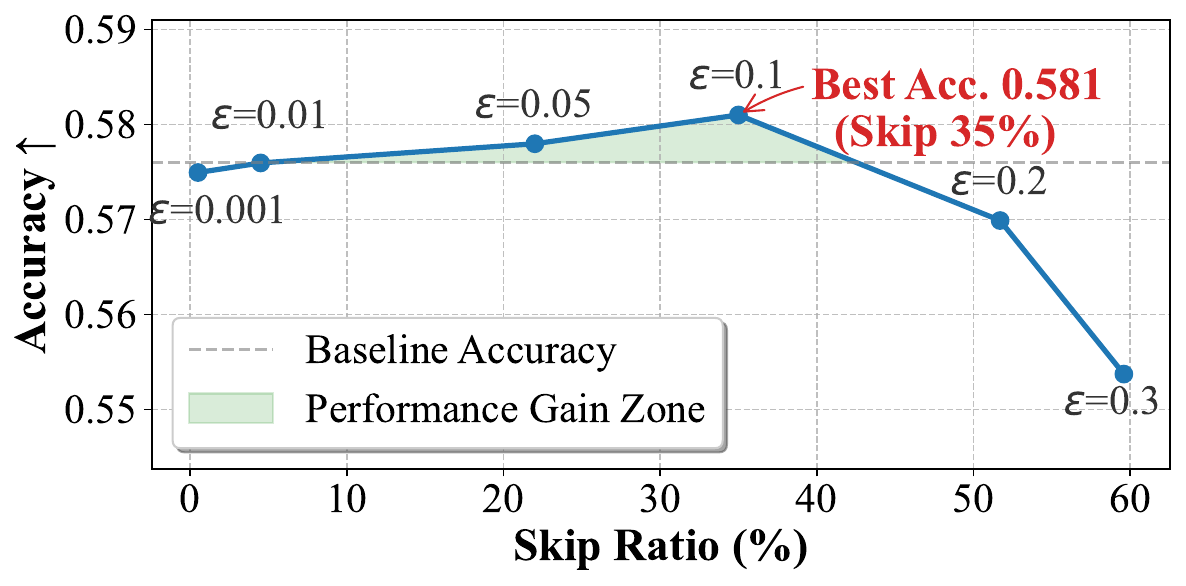}
        \caption{Trade off between zero-shot accuracy and the expert skip ratio.}
        \label{fig:epsilon_tradeoff}
    \end{subfigure}
    \begin{subfigure}[t]{0.48\textwidth}
        \centering
        \includegraphics[width=\linewidth]{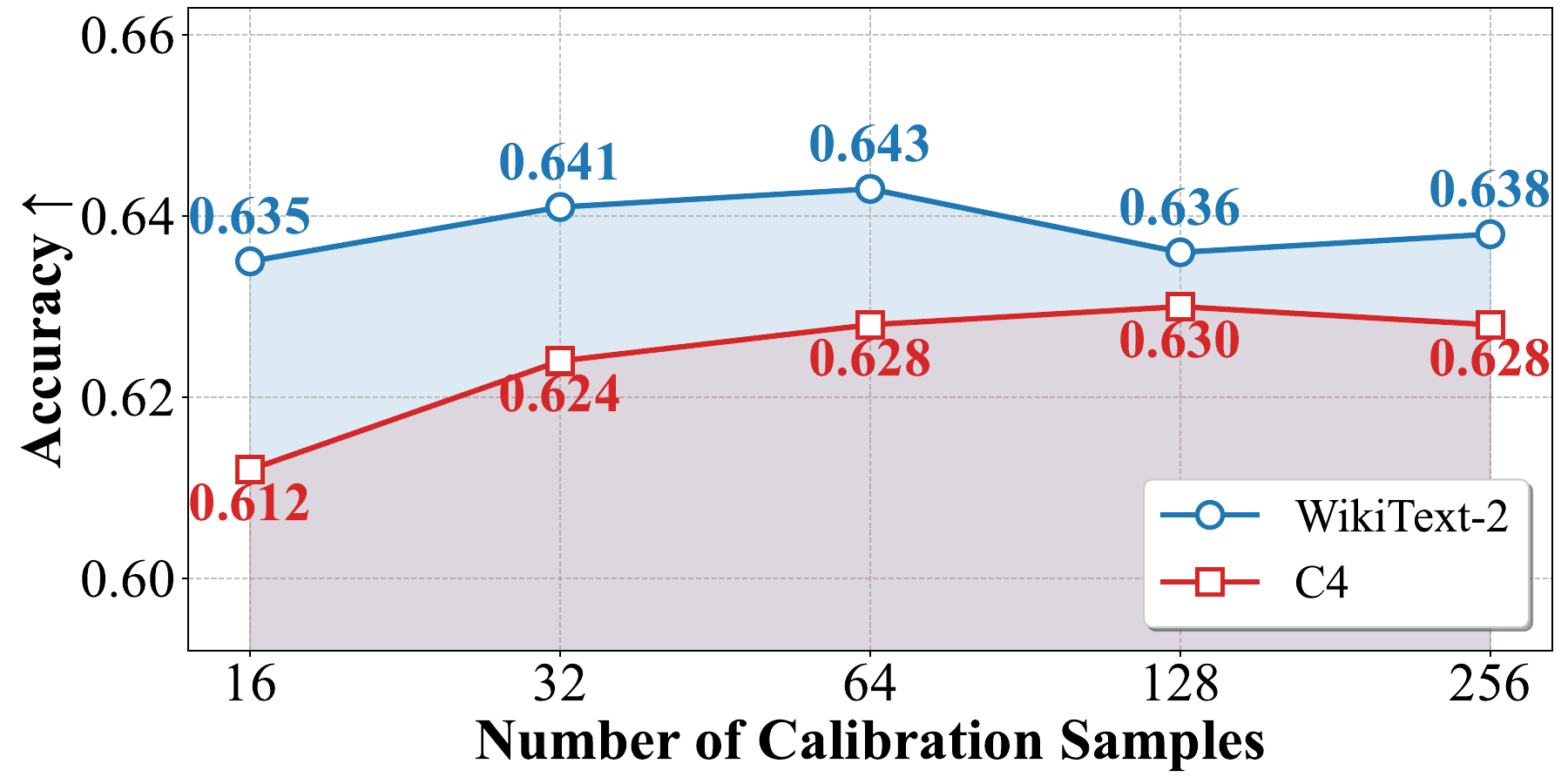}
        \caption{Calibration sensitivity.}
        \label{fig:calibration_analysis}
    \end{subfigure}
    \caption{Ablations across eight benchmarks. (a) Accuracy and expert skipping under different adaptive thresholds on 25\%-pruned DeepSeek-V2-Lite. (b) Accuracy under different calibration sizes and domains on 50\%-pruned Qwen3-30B-A3B.}
\label{fig:skip_and_calib}
\vspace{-10pt}
\end{figure}

\textbf{Impact of Calibration Configurations.}
Figure~\ref{fig:skip_and_calib}(b) studies calibration size and domain on 50\%-pruned Qwen3-30B-A3B. WikiText-2 peaks at 0.643 with 64 samples, while C4 reaches 0.630 with 128. Additional samples provide little benefit. Appendix~\ref{app:reconstruction_accuracy} further shows a positive relation between normalized reconstruction fidelity and downstream accuracy. These results indicate that a modest calibration set is sufficient, although the calibration domain can affect the attainable accuracy.

\section{Conclusion}

Existing MoE compression methods rely on costly expert selection in a combinatorial search space and use uniform and fixed resource budgets across layers and inputs. This design ignores the structural heterogeneity of MoE layers and leads to suboptimal compression and unnecessary inference costs. We present OMP-MoE as a new framework for MoE compression that reformulates expert selection as a sparse signal reconstruction task. This formulation reduces the search space from combinatorial to linear and significantly accelerates the search process. We also introduce cross-layer budget allocation, which assigns compression budgets according to the redundancy of each layer, and OMP-MoE$^\dagger$ adaptive inference, which dynamically adjusts computation for different inputs. These innovations improve compression effectiveness and reduce inference costs. Comprehensive experiments across diverse MoE models validate our framework. Its training-free operation and consistent effectiveness across architectures make OMP-MoE a practical solution for deploying large-scale MoE LLMs.

\textbf{Limitations.} Following most MoE compressors, our technique relies on calibration dataset selection, and we will explore advanced calibration procedures~\cite{shi2025dids} in future work. Future work could combine OMP-MoE with complementary compression and adaptation techniques~\cite{li2025nora,li2025aira,zhu2025oamerge,zhu2026outlier} to further improve efficiency while preserving domain-specific capabilities. Finally, the theoretical connection to OMP offers a promising basis for studying how expert selection relates to model interpretability and robustness under distribution shifts or perturbations.
\clearpage

\bibliographystyle{plain}
\bibliography{OMP-MoE}

@article{jiang2024mixtral,
  title={Mixtral of experts},
  author={Jiang, Albert Q and Sablayrolles, Alexandre and Roux, Antoine and Mensch, Arthur and Savary, Blanche and Bamford, Chris and Chaplot, Devendra Singh and Casas, Diego de las and Hanna, Emma Bou and Bressand, Florian and others},
  journal={arXiv preprint arXiv:2401.04088},
  year={2024}
}

@article{deepseekai2024deepseekv3,
  title={DeepSeek-V3 Technical Report},
  author={DeepSeek-AI},
  journal={arXiv preprint arXiv:2412.19437},
  year={2024}
}

@article{su2025superexperts,
  title={Unveiling Super Experts in Mixture-of-Experts Large Language Models},
  author={Su, Zunhai and Li, Qingyuan and Zhang, Hao and Ye, Weihao and Xue, Qibo and Qian, Yulei and Xie, Yuchen and Wong, Ngai and Yuan, Kehong},
  journal={arXiv preprint arXiv:2507.23279},
  year={2025}
}

@article{yang2025qwen3,
  title={Qwen3 technical report},
  author={Yang, An and Li, Anfeng and Yang, Baosong and Zhang, Beichen and Hui, Binyuan and Zheng, Bo and Yu, Bowen and Gao, Chang and Huang, Chengen and Lv, Chenxu and others},
  journal={arXiv preprint arXiv:2505.09388},
  year={2025}
}

@article{chen2021codex,
  title={Evaluating Large Language Models Trained on Code},
  author={Mark Chen and Jerry Tworek and Heewoo Jun and Qiming Yuan and Henrique Ponde de Oliveira Pinto and Jared Kaplan and Harri Edwards and Yuri Burda and Nicholas Joseph and Greg Brockman and Alex Ray and Raul Puri and Gretchen Krueger and Michael Petrov and Heidy Khlaaf and Girish Sastry and Pamela Mishkin and Brooke Chan and Scott Gray and Nick Ryder and Mikhail Pavlov and Alethea Power and Lukasz Kaiser and Mohammad Bavarian and Clemens Winter and Philippe Tillet and Felipe Petroski Such and Dave Cummings and Matthias Plappert and Fotios Chantzis and Elizabeth Barnes and Ariel Herbert-Voss and William Hebgen Guss and Alex Nichol and Alex Paino and Nikolas Tezak and Jie Tang and Igor Babuschkin and Suchir Balaji and Shantanu Jain and William Saunders and Christopher Hesse and Andrew N. Carr and Jan Leike and Josh Achiam and Vedant Misra and Evan Morikawa and Alec Radford and Matthew Knight and Miles Brundage and Mira Murati and Katie Mayer and Peter Welinder and Bob McGrew and Dario Amodei and Sam McCandlish and Ilya Sutskever and Wojciech Zaremba},
  journal={arXiv preprint arXiv:2107.03374},
  year={2021}
}

@inproceedings{lu2024naee,
    title = "Not All Experts are Equal: Efficient Expert Pruning and Skipping for Mixture-of-Experts Large Language Models",
    author = "Lu, Xudong  and
      Liu, Qi  and
      Xu, Yuhui  and
      Zhou, Aojun  and
      Huang, Siyuan  and
      Zhang, Bo  and
      Yan, Junchi  and
      Li, Hongsheng",
    booktitle = "Proceedings of the 62nd Annual Meeting of the ACL",
    month = "Aug",
    year = "2024",
    address = "Bangkok, Thailand",
    publisher = "Association for Computational Linguistics",
    pages = "6159--6172",
}

@inproceedings{
li2023mcsmoe,
title={Merge, Then Compress: Demystify Efficient {SM}oE with Hints from Its Routing Policy},
author={Pingzhi Li and Zhenyu Zhang and Prateek Yadav and Yi-Lin Sung and Yu Cheng and Mohit Bansal and Tianlong Chen},
booktitle={International Conference on Learning Representations},
year={2024},
url={https://openreview.net/forum?id=eFWG9Cy3WK}
}

@inproceedings{
chen2024hcsmoe,
title={Retraining-free Merging of Sparse MoE via Hierarchical Clustering},
author={I-Chun Chen and Hsu-Shen Liu and Wei-Fang Sun and Chen-Hao Chao and Yen-Chang Hsu and Chun-Yi Lee},
booktitle={International Conference on Machine Learning},
year={2025},
url={https://openreview.net/forum?id=hslOzRxzXL}
}

@InProceedings{frantar2023sparsegpt,
  title = 	 {{S}parse{GPT}: Massive Language Models Can be Accurately Pruned in One-Shot},
  author =       {Frantar, Elias and Alistarh, Dan},
  booktitle = 	 {International Conference on Machine Learning},
  pages = 	 {10323--10337},
  year = 	 {2023},
  volume = 	 {202},
  month = 	 {Jul},
  publisher =    {PMLR},
}

@inproceedings{sun2024wanda,
  title={A Simple and Effective Pruning Approach for Large Language Models},
  author={Mingjie Sun and Zhuang Liu and Anna Bair and J Zico Kolter},
  booktitle={International Conference on Learning Representations},
  year={2024},
  url={https://openreview.net/forum?id=PxoFut3dWW}
}

@article{li2025heapr,
  title={HEAPr: Hessian-based Efficient Atomic Expert Pruning in Output Space},
  author={Li, Ke and Yang, Zheng and Zhou, Zhongbin and Xue, Feng and Jiang, Zhonglin and Wang, Wenxiao},
  journal={arXiv preprint arXiv:2509.22299},
  year={2025}
}

@inproceedings{pati1993orthogonal,
  title={Orthogonal matching pursuit: Recursive function approximation with applications to wavelet decomposition},
  author={Pati, Yagyensh Chandra and Rezaiifar, Ramin and Krishnaprasad, Perinkulam Sambamurthy},
  booktitle={Proceedings of 27th Asilomar conference on signals, systems and computers},
  pages={40--44},
  year={1993},
  organization={IEEE}
}

@article{tropp2007signal,
  title={Signal recovery from random measurements via orthogonal matching pursuit},
  author={Tropp, Joel A and Gilbert, Anna C},
  journal={IEEE Transactions on information theory},
  volume={53},
  number={12},
  pages={4655--4666},
  year={2007},
  publisher={IEEE}
}

@inproceedings{lasby2025reap,
 author = {Lasby, Mike and Lazarevich, Ivan and Sinnadurai, Nish and Lie, Sean and Ioannou, Yani and Thangarasa, Vithursan},
 booktitle = {International Conference on Learning Representations},
 pages = {146883--146912},
 title = {REAP the Experts: Why Pruning Prevails for One-Shot MoE compression},
 volume = {2026},
 year = {2026}
}

@article{shazeer2017outrageously,
  title={Outrageously large neural networks: The sparsely-gated mixture-of-experts layer},
  author={Shazeer, Noam and Mirhoseini, Azalia and Maziarz, Krzysztof and Davis, Andy and Le, Quoc and Hinton, Geoffrey and Dean, Jeff},
  journal={arXiv preprint arXiv:1701.06538},
  year={2017}
}

@article{li2025submoe,
  title={Sub-MoE: Efficient Mixture-of-Expert LLMs Compression via Subspace Expert Merging},
  author={Li, Lujun and Qiyuan, Zhu and Wang, Jiacheng and Li, Wei and Gu, Hao and Han, Sirui and Guo, Yike},
  journal={arXiv preprint arXiv:2506.23266},
  year={2025}
}

@inproceedings{
lepikhin2021gshard,
title={{\{}GS{\}}hard: Scaling Giant Models with Conditional Computation and Automatic Sharding},
author={Dmitry Lepikhin and HyoukJoong Lee and Yuanzhong Xu and Dehao Chen and Orhan Firat and Yanping Huang and Maxim Krikun and Noam Shazeer and Zhifeng Chen},
booktitle={International Conference on Learning Representations},
year={2021},
url={https://openreview.net/forum?id=qrwe7XHTmYb}
}

@article{suzgun2022bbh,
  title={Challenging BIG-Bench Tasks and Whether Chain-of-Thought Can Solve Them},
  author={Suzgun, Mirac and Scales, Nathan and Sch{\"a}rli, Nathanael and Gehrmann, Sebastian and Tay, Yi and Chung, Hyung Won and Chowdhery, Aakanksha and Le, Quoc V and Chi, Ed H and Zhou, Denny and and Wei, Jason},
  journal={arXiv preprint arXiv:2210.09261},
  year={2022}
}

@inproceedings{
kim2025lexico,
title={Lexico: Extreme {KV} Cache Compression via Sparse Coding over Universal Dictionaries},
author={Junhyuck Kim and Jongho Park and Jaewoong Cho and Dimitris Papailiopoulos},
booktitle={International Conference on Machine Learning},
year={2025},
url={https://openreview.net/forum?id=Yh9vxlxnjA}
}

@INPROCEEDINGS{Yang2009CNN,
  author={Yang, Jie and Bouzerdoum, Abdesselam and Phung, Son Lam},
  booktitle={International Joint Conference on Neural Networks}, 
  title={A Neural Network pruning approach based on Compressive Sampling}, 
  year={2009},
  pages={3428-3435},
}

@inproceedings{
huang2025shapleymoe,
title={Discovering Important Experts for Mixture-of-Experts Models Pruning Through a Theoretical Perspective},
author={Weizhong Huang and Yuxin Zhang and Xiawu Zheng and Fei Chao and Rongrong Ji and Liujuan Cao},
booktitle={Advances in Neural Information Processing Systems},
year={2025},
}

@article{liu2024deepseekv2,
  title={Deepseek-v2: A strong, economical, and efficient mixture-of-experts language model},
  author={Liu, Aixin and Feng, Bei and Wang, Bin and Wang, Bingxuan and Liu, Bo and Zhao, Chenggang and Dengr, Chengqi and Ruan, Chong and Dai, Damai and Guo, Daya and others},
  journal={arXiv preprint arXiv:2405.04434},
  year={2024}
}

@article{agarwal2025gptoss,
  title={gpt-oss-120b \& gpt-oss-20b model card},
  author={Agarwal, Sandhini and Ahmad, Lama and Ai, Jason and Altman, Sam and Applebaum, Andy and Arbus, Edwin and Arora, Rahul K and Bai, Yu and Baker, Bowen and Bao, Haiming and others},
  journal={arXiv preprint arXiv:2508.10925},
  year={2025}
}

@article{fedus2022switch,
  title={Switch transformers: Scaling to trillion parameter models with simple and efficient sparsity},
  author={Fedus, William and Zoph, Barret and Shazeer, Noam},
  journal={Journal of Machine Learning Research},
  volume={23},
  number={120},
  pages={1--39},
  year={2022}
}

@inproceedings{
frantar2023gptq,
title={{OPTQ}: Accurate Quantization for Generative Pre-trained Transformers},
author={Elias Frantar and Saleh Ashkboos and Torsten Hoefler and Dan Alistarh},
booktitle={The Eleventh International Conference on Learning Representations },
year={2023},
url={https://openreview.net/forum?id=tcbBPnfwxS}
}

@inproceedings{
hendrycks2021mmlu,
title={Measuring Massive Multitask Language Understanding},
author={Dan Hendrycks and Collin Burns and Steven Basart and Andy Zou and Mantas Mazeika and Dawn Song and Jacob Steinhardt},
booktitle={International Conference on Learning Representations},
year={2021},
url={https://openreview.net/forum?id=d7KBjmI3GmQ}
}

@article{zhang2025mone,
  title={MoNE: Replacing Redundant Experts with Lightweight Novices for Structured Pruning of MoE},
  author={Zhang, Geng and Han, Yuxuan and Lou, Yuxuan and Zhao, Wangbo and Zhang, Yiqi and You, Yang},
  journal={arXiv preprint arXiv:2507.00390},
  year={2025}
}

@article{bai2025diep,
  title={DiEP: Adaptive Mixture-of-Experts Compression through Differentiable Expert Pruning},
  author={Bai, Sikai and Li, Haoxi and Zhang, Jie and Hong, Zicong and Guo, Song},
  journal={arXiv preprint arXiv:2509.16105},
  year={2025}
}

@inproceedings{bai2024longbench,
    title = "{L}ong{B}ench: A Bilingual, Multitask Benchmark for Long Context Understanding",
    author = "Bai, Yushi  and
      Lv, Xin  and
      Zhang, Jiajie  and
      Lyu, Hongchang  and
      Tang, Jiankai  and
      Huang, Zhidian  and
      Du, Zhengxiao  and
      Liu, Xiao  and
      Zeng, Aohan  and
      Hou, Lei  and
      Dong, Yuxiao  and
      Tang, Jie  and
      Li, Juanzi",
    booktitle = "Proceedings of the Association for Computational Linguistics",
    month = "Aug",
    year = "2024",
    address = "Bangkok, Thailand",
    publisher = "Association for Computational Linguistics",
    pages = "3119--3137",
}

@article{cobbe2021gsm8k,
  title={Training Verifiers to Solve Math Word Problems},
  author={Cobbe, Karl and Kosaraju, Vineet and Bavarian, Mohammad and Chen, Mark and Jun, Heewoo and Kaiser, Lukasz and Plappert, Matthias and Tworek, Jerry and Hilton, Jacob and Nakano, Reiichiro and Hesse, Christopher and Schulman, John},
  journal={arXiv preprint arXiv:2110.14168},
  year={2021}
}

@inproceedings{Bentivogli2009rte,
  author       = {Luisa Bentivogli and
                  Bernardo Magnini and
                  Ido Dagan and
                  Hoa Trang Dang and
                  Danilo Giampiccolo},
  title        = {The Fifth {PASCAL} Recognizing Textual Entailment Challenge},
  booktitle    = {Proceedings of the TAC 2009, Gaithersburg,
                  Maryland, USA, November 16-17, 2009},
  publisher    = {{NIST}},
  year         = {2009},
}

@inproceedings{clark2019boolq,
    title = "{B}ool{Q}: Exploring the Surprising Difficulty of Natural Yes/No Questions",
    author = "Clark, Christopher  and
      Lee, Kenton  and
      Chang, Ming-Wei  and
      Kwiatkowski, Tom  and
      Collins, Michael  and
      Toutanova, Kristina",
    booktitle = "Proceedings of the Association for Computational Linguistics",
    month = "June",
    year = "2019",
    address = "Minneapolis, Minnesota",
    publisher = "Association for Computational Linguistics",
    pages = "2924--2936",
}

@article{merity2017wikitext2,
  title={Pointer sentinel mixture models},
  author={Merity, Stephen and Xiong, Caiming and Bradbury, James and Socher, Richard},
  journal={arXiv preprint arXiv:1609.07843},
  year={2017}
}

@article{raffel2020c4,
  title={Exploring the limits of transfer learning with a unified text-to-text transformer},
  author={Raffel, Colin and Shazeer, Noam and Roberts, Adam and Lee, Katherine and Narang, Sharan and Matena, Michael and Zhou, Yanqi and Li, Wei and Liu, Peter J},
  journal={Journal of Machine Learning Research},
  volume={21},
  number={140},
  pages={1--67},
  year={2020}
}

@inproceedings{mihaylov2018openbookqa,
    title = "Can a Suit of Armor Conduct Electricity? A New Dataset for Open Book Question Answering",
    author = "Mihaylov, Todor  and
      Clark, Peter  and
      Khot, Tushar  and
      Sabharwal, Ashish",
    booktitle = "Proceedings of the Conference on Empirical Methods in Natural Language Processing",
    month = "Oct",
    year = "2018",
    address = "Brussels, Belgium",
    publisher = "Association for Computational Linguistics",
    pages = "2381--2391",
}

@article{sakaguchi2020winogrande,
  title={WinoGrande: An adversarial winograd schema challenge at scale},
  author={Sakaguchi, Keisuke and Bras, Ronan Le and Bhagavatula, Chandra and Choi, Yejin},
  journal={Communications of the ACM},
  volume={64},
  number={9},
  pages={99--106},
  year={2020}
}

@inproceedings{zellers2019hellaswag,
    title = "{H}ella{S}wag: Can a Machine Really Finish Your Sentence?",
    author = "Zellers, Rowan  and
      Holtzman, Ari  and
      Bisk, Yonatan  and
      Farhadi, Ali  and
      Choi, Yejin",
    booktitle = "Proceedings of the Association for Computational Linguistics",
    month = "Jul",
    year = "2019",
    address = "Florence, Italy",
    publisher = "Association for Computational Linguistics",
    pages = "4791--4800",
}

@article{clark2018arc,
  title={Think you have solved question answering? Try ARC, the AI2 reasoning challenge},
  author={Clark, Peter and Cowhey, Isaac and Etzioni, Oren and Khot, Tushar and Sabharwal, Ashish and Schoenick, Carissa and Tafjord, Oyvind},
  journal={arXiv preprint arXiv:1803.05457},
  year={2018}
}

@inproceedings{dai2022stablemoe,
    title = "{S}table{M}o{E}: Stable Routing Strategy for Mixture of Experts",
    author = "Dai, Damai  and
      Dong, Li  and
      Ma, Shuming  and
      Zheng, Bo  and
      Sui, Zhifang  and
      Chang, Baobao  and
      Wei, Furu",
    booktitle = "Proceedings of the Association for Computational Linguistics",
    month = "May",
    year = "2022",
    address = "Dublin, Ireland",
    publisher = "Association for Computational Linguistics",
    pages = "7085--7095",
}

@article{report2023gpt4,
  title={Gpt-4 technical report},
  author={Achiam, Josh and Adler, Steven and Agarwal, Sandhini and Ahmad, Lama and Akkaya, Ilge and Aleman, Florencia Leoni and Almeida, Diogo and Altenschmidt, Janko and Altman, Sam and Anadkat, Shyamal and others},
  journal={arXiv preprint arXiv:2303.08774},
  year={2023}
}

@article{report2024llama3,
  title={The llama 3 herd of models},
  author={Grattafiori, Aaron and Dubey, Abhimanyu and Jauhri, Abhinav and Pandey, Abhinav and Kadian, Abhishek and Al-Dahle, Ahmad and Letman, Aiesha and Mathur, Akhil and Schelten, Alan and Vaughan, Alex and others},
  journal={arXiv preprint arXiv:2407.21783},
  year={2024}
}

@inproceedings{
ghorbani2022scaling,
title={Scaling Laws for Neural Machine Translation},
author={Behrooz Ghorbani and Orhan Firat and Markus Freitag and Ankur Bapna and Maxim Krikun and Xavier Garcia and Ciprian Chelba and Colin Cherry},
booktitle={International Conference on Learning Representations},
year={2022},
url={https://openreview.net/forum?id=hR_SMu8cxCV}
}

@InProceedings{Clark2022ScalingMoe,
  title = 	 {Unified Scaling Laws for Routed Language Models},
  author =       {Clark, Aidan and De Las Casas, Diego and Guy, Aurelia and Mensch, Arthur and Paganini, Michela and Hoffmann, Jordan and Damoc, Bogdan and Hechtman, Blake and Cai, Trevor and Borgeaud, Sebastian and Van Den Driessche, George Bm and Rutherford, Eliza and Hennigan, Tom and Johnson, Matthew J and Cassirer, Albin and Jones, Chris and Buchatskaya, Elena and Budden, David and Sifre, Laurent and Osindero, Simon and Vinyals, Oriol and Ranzato, Marc'Aurelio and Rae, Jack and Elsen, Erich and Kavukcuoglu, Koray and Simonyan, Karen},
  booktitle = 	 {Proceedings of the International Conference on Machine Learning},
  pages = 	 {4057--4086},
  year = 	 {2022},
  volume = 	 {162},
  month = 	 {17--23 Jul},
  publisher =    {PMLR},
}

@inproceedings{li2025nora,
  author    = {Li, Lujun and Lin, Cheng and Li, Dezhi and Huang, You-Liang and Li, Wei and Wu, Tianyu and Zou, Jie and Xue, Wei and Han, Sirui and Guo, Yike},
  title     = {Efficient Fine-Tuning of Large Models Via Nested Low-Rank Adaptation},
  booktitle = {IEEE/CVF International Conference on Computer Vision (ICCV)},
  pages     = {22252--22262},
  publisher = {IEEE},
  year      = {2025},
  address   = {Honolulu, HI, USA},
  doi       = {10.1109/ICCV51701.2025.02066},
}

@inproceedings{chen2026adaptive,
  author    = {Chen, Tianle and Cheng, Pengyu and Zhu, Qiyuan and Wang, Jiacheng and Liu, Bei and Gu, Hao and Shen, Ruijie and Hou, Xiaofeng and Han, Sirui and Liu, Jiacheng},
  title     = {Adaptive Spatial and Temporal Redundancy Optimization for Efficient Reasoning in Large Language Models},
  booktitle = {Proceedings of the 64th Annual Meeting of the Association for Computational Linguistics (Volume 1: Long Papers)},
  pages     = {24647--24662},
  publisher = {Association for Computational Linguistics},
  address   = {San Diego, California, United States},
  year      = {2026},
  doi       = {10.18653/v1/2026.acl-long.1130},
}

@inproceedings{li2025aira,
  author    = {Li, Lujun and Li, Dezhi and Lin, Cheng and Li, Wei and Xue, Wei and Han, Sirui and Guo, Yike},
  title     = {{AIRA}: Activation-Informed Low-Rank Adaptation for Large Models},
  booktitle = {IEEE/CVF International Conference on Computer Vision (ICCV)},
  pages     = {1729--1739},
  publisher = {IEEE},
  year      = {2025},
  address   = {Honolulu, HI, USA},
  doi       = {10.1109/ICCV51701.2025.00169},
}

@article{hu2026hierarchical,
  title={Hierarchical Sparse Attention Done Right: Toward Infinite Context Modeling},
  author={Hu, Xiang and Wei, Xinyu and Gu, Hao and Zhang, Minshen and Liang, Tian and Li, Huayang and Zhu, Lei and Wang, Yan and Han, Sirui and Bai, Yushi and Tu, Kewei and Mi, Haitao and Liang, Leo},
  journal={arXiv preprint arXiv:2607.02980},
  year={2026}
}

@inproceedings{zhu2025oamerge,
author = {Zhu, Qiyuan and Li, Lujun and Li, Dezhi and Liu, Jiacheng and Cheng, Pengyu and Xu, Yucheng and Han, Sirui and Guo, Yike},
title = {Outlier-Aware Model Merging for Efficient Multitask Inference},
year = {2025},
publisher = {Association for Computing Machinery},
address = {Dublin, Ireland},
booktitle = {Proceedings of the ACM International Conference on Multimedia},
pages = {1032--1041},
}

@article{zhu2026outlier,
  title={Outlier Matters: Efficient Long-to-Short Reasoning via Outlier-Guided Model Merging},
  author={Zhu, Qiyuan and Li, Dezhi and Li, Lujun and Qin, Xiaoyu and Li, Wei and Gu, Hao and Xu, Hua and Han, Sirui and Guo, Yike},
  journal={Proceedings of the AAAI Conference on Artificial Intelligence},
  address={Singapore},
  volume={40},
  number={41},
  pages={35213--35221},
  month={Jan},
  year={2026}
}

@inproceedings{shi2025dids,
    title = "{DIDS}: Domain Impact-aware Data Sampling for Large Language Model Training",
    author = "Shi, Weijie  and
      Zhang, Jipeng  and
      Wu, Yaguang  and
      Fang, Jingzhi  and
      Zhang, Shibo  and
      Zhao, Yao  and
      Chen, Hao  and
      Zhang, Ruiyuan  and
      Cui, Yue  and
      Zhu, Jia  and
      Han, Sirui  and
      Xu, Jiajie  and
      Zhou, Xiaofang",
    booktitle = "Proceedings of the Conference on Empirical Methods in Natural Language Processing",
    month = "Nov",
    year = "2025",
    address = "Suzhou, China",
    publisher = "Association for Computational Linguistics",
    pages = "4330--4350",
}

@inproceedings{Dong2025easyep,
 author = {Dong, Zican and Peng, Han and Liu, Peiyu and Zhao, Xin and Wu, Dong and Xiao, Feng and Wang, Zhifeng},
 booktitle = {Advances in Neural Information Processing Systems},
 pages = {103552--103577},
 publisher = {Curran Associates, Inc.},
 address = {San Diego, CA, USA},
 title = {Domain-Specific Pruning of Large Mixture-of-Experts Models with Few-shot Demonstrations},
 volume = {38, Main Conference},
 year = {2025}
}

@article{zhu2026fewer,
  title={Fewer Tokens, Smaller Cache: Reward-Coordinated Efficient Reasoning},
  author={Zhu, Qiyuan and Li, Dezhi and Cheng, Pengyu and Chen, Tianle and Wang, Jiacheng and Shen, Ruijie and Gu, Hao and Lin, Sida and Liu, Zirui and Liu, Jiacheng and Han, Sirui},
  journal={arXiv preprint arXiv:2608.04771},
  year={2026}
}

@inproceedings{xu2026bit,
    title = "Bit-by-Bit: Progressive {QAT} Strategy with Outlier Channel Splitting for Stable Low-Bit {LLMs}",
    author = "Xu, Binxing  and
      Gu, Hao  and
      Li, Lujun  and
      Wang, Hao  and
      Liu, Bei  and
      Liu, Jiacheng  and
      Zhu, Qiyuan  and
      Yang, Xintong  and
      Li, Chao  and
      Han, Sirui  and
      Guo, Yike",
    booktitle = "Proceedings of the 64th Annual Meeting of the ACL",
    month = "Jul",
    year = "2026",
    address = "San Diego, California, United States",
    publisher = "Association for Computational Linguistics",
    pages = "22284--22299",
}

\newpage
\appendix

\section{Formal Theoretical Analysis of OMP-MoE}
\label{app:formal_analysis}
In this section, we provide a rigorous mathematical treatment of the expert selection process. We frame the Mixture-of-Experts (MoE) pruning problem within the context of Hilbert space theory and prove the optimality of our greedy selection criterion under the linear decomposition property.

\subsection{Mathematical Framework}

Let $\mathcal{H} = \mathbb{R}^{B \times d}$ be the Hilbert space of tensors equipped with the Frobenius inner product $\langle \mathbf{A}, \mathbf{B} \rangle_{\mathcal{H}} = \text{Tr}(\mathbf{A}^\top \mathbf{B})$ and the induced norm $\|\mathbf{A}\|_{\mathcal{H}} = \sqrt{\langle \mathbf{A}, \mathbf{A} \rangle_{\mathcal{H}}}$.

In the standard Orthogonal Matching Pursuit (OMP) formulation, the objective is to approximate a signal $\mathbf{Y}$ using a linear combination of atoms from a dictionary $\mathcal{D} = \{\mathbf{V}_e\}_{e=1}^{N_e}$ with continuous coefficients $\mathbf{c} \in \mathbb{R}^{N_e}$. However, in the context of MoE pruning, ``partial selection'' of an expert is computationally invalid, because an expert is either executed (coefficient 1) or pruned (coefficient 0).

We formally define the \textit{Expert Selection Problem} as a constrained optimization over the Boolean hypercube:
\begin{equation}
    \min_{\mathbf{c} \in \{0, 1\}^{N_e}} \left\| \mathbf{Y} - \sum_{e=1}^{N_e} c_e \mathbf{V}_e \right\|_{\mathcal{H}}^2 \quad \text{subject to} \quad \|\mathbf{c}\|_0 = n.
\end{equation}
Let $E = \{e \mid c_e = 1\}$ be the support set. The problem reduces to finding the optimal subset indices:
\begin{equation}
    \min_{E: |E|=n} \mathcal{J}(E) \coloneqq \left\| \mathbf{Y} - \sum_{e \in E} \mathbf{V}_e \right\|_{\mathcal{H}}^2.
\end{equation}

\subsection{Derivation of the Greedy Selection Rule}\label{app:binary_greedy_optimality}

Unlike standard OMP, which requires an orthogonal projection of the signal onto the subspace spanned by the selected atoms (typically involving the inversion of the Gram matrix), our binary constraint simplifies the update dynamics significantly.

\begin{theorem}[Optimality of Greedy Selection Criterion]
Given a current residual $\mathbf{R}_t = \mathbf{Y} - \sum_{e \in E_t} \mathbf{V}_e$ at iteration $t$, and imposing the binary constraint $c_{e^*} \in \{0, 1\}$, the optimal next expert $e^*$ to add to the support set $E_t$ is determined by:
\begin{equation}
    e^* = \arg\max_{e \notin E_t} \left\{ 2\langle \mathbf{R}_t, \mathbf{V}_e \rangle_{\mathcal{H}} - \|\mathbf{V}_e\|_{\mathcal{H}}^2 \right\}.
\end{equation}
Furthermore, the residual update is strictly additive: $\mathbf{R}_{t+1} = \mathbf{R}_t - \mathbf{V}_{e^*}$.
\end{theorem}

\begin{proof}
Let $e \notin E_t$ be a candidate expert. Under the binary constraint, if expert $e$ is selected, its coefficient is fixed at $1$. Thus, the new residual is deterministically defined as $\mathbf{R}_{t+1} = \mathbf{R}_t - \mathbf{V}_e$. We analyze the energy reduction (squared norm) of the residual:
\begin{align}
    \|\mathbf{R}_{t+1}\|_{\mathcal{H}}^2 &= \|\mathbf{R}_t - \mathbf{V}_e\|_{\mathcal{H}}^2 \nonumber \\
    &= \langle \mathbf{R}_t - \mathbf{V}_e, \mathbf{R}_t - \mathbf{V}_e \rangle_{\mathcal{H}} \nonumber \\
    &= \langle \mathbf{R}_t, \mathbf{R}_t \rangle_{\mathcal{H}} - 2\langle \mathbf{R}_t, \mathbf{V}_e \rangle_{\mathcal{H}} + \langle \mathbf{V}_e, \mathbf{V}_e \rangle_{\mathcal{H}} \nonumber \\
    &= \|\mathbf{R}_t\|_{\mathcal{H}}^2 - \underbrace{\left( 2\langle \mathbf{R}_t, \mathbf{V}_e \rangle_{\mathcal{H}} - \|\mathbf{V}_e\|_{\mathcal{H}}^2 \right)}_{\Delta(e; \mathbf{R}_t)}.
\end{align}
Since $\|\mathbf{R}_t\|_{\mathcal{H}}^2$ is constant with respect to $e$, minimizing the new residual norm $\|\mathbf{R}_{t+1}\|_{\mathcal{H}}^2$ is equivalent to maximizing the term $\Delta(e; \mathbf{R}_t)$. This proves that the greedy selection based on $\Delta$ is locally optimal under the binary constraint.
\end{proof}

\subsection{Comparison with Standard OMP Projection}

It is crucial to distinguish this update rule from standard OMP. In standard OMP, the residual is updated as $\mathbf{R}_{t+1}^{std} = (\mathbf{I} - \mathbf{P}_{t+1})\mathbf{Y}$, where $\mathbf{P}_{t+1}$ is the orthogonal projector onto $\text{span}(\{\mathbf{V}_e\}_{e \in E_{t+1}})$. This implies that previously selected experts re-adjust their coefficients to compensate for the new atom.

In OMP-MoE, the binary constraint forbids such re-adjustment (coefficients must remain 1). Consequently, the ``orthogonalization'' step is replaced by a direct subtraction. This yields two theoretical implications:
\begin{enumerate}
    \item \textbf{Computational Efficiency:} We avoid the $\mathcal{O}(n^3)$ cost of matrix inversion required for projection, reducing the complexity to simple tensor additions.
    \item \textbf{Independence of Selection:} The contribution of an expert $\mathbf{V}_e$ is invariant to the selection of future experts, preserving the physical interpretation of the MoE routing mechanism where expert weights are determined by the router, not by a post-hoc least-squares fit.
\end{enumerate}

\subsection{Budget Allocation via Marginal Gain}

The cross-layer budget allocation involves distributing a total budget $C_{\text{tot}}$ across $L$ layers. Let $f_l(k) = \|\mathbf{R}_{l,k}\|_{\mathcal{H}}^2$ be the minimum residual norm for layer $l$ using $k$ experts.

\begin{theorem}[Optimality of Greedy Allocation]
If the sequence of marginal gains $\delta_l(k) = f_l(k-1) - f_l(k)$ is monotonically non-increasing for each layer $l$ (i.e., $f_l$ is discrete convex), then the greedy allocation strategy yields the globally optimal expert distribution $\{k_l^*\}_{l=1}^L$ that minimizes the total reconstruction error $\sum_{l=1}^L f_l(k_l)$.
\end{theorem}

\begin{proof}
This follows from the theory of submodular optimization. The objective function $F(k_1, \dots, k_L) = \sum f_l(k_l)$ is a separable function. If each $f_l$ exhibits diminishing marginal returns (a property we empirically verify in Figure \ref{fig:analysis_residual}), the greedy algorithm that at each step allocates one unit of budget to the layer $l = \arg\max \delta_l(k_l+1)$ is guaranteed to reach the global optimum for the resource allocation problem under a matroid constraint.
\end{proof}

\subsection{Reconstruction Error and Loss Perturbation}\label{app:sensitivity_energy}

Let $\mathbf{Y}_l$ be the original output of layer $l$, and let the pruned or skipped output be
\begin{equation}
\tilde{\mathbf{Y}}_l = \mathbf{Y}_l - \mathbf{R}_l,
\end{equation}
where $\mathbf{R}_l$ is the reconstruction residual. For a downstream loss $\mathcal{L}$, a local second-order expansion gives
\begin{align}
\Delta \mathcal{L}
&= \mathcal{L}(\{\tilde{\mathbf{Y}}_l\}_{l=1}^{L})-\mathcal{L}(\{\mathbf{Y}_l\}_{l=1}^{L}) \nonumber \\
&\approx -\sum_{l=1}^{L}\langle \nabla_{\mathbf{Y}_l}\mathcal{L}, \mathbf{R}_l\rangle \nonumber \\
&\qquad +\frac{1}{2}\sum_{l=1}^{L}\mathrm{vec}(\mathbf{R}_l)^\top \mathbf{H}_l \mathrm{vec}(\mathbf{R}_l).
\end{align}
Assume bounded local sensitivity:
\begin{equation}
\|\nabla_{\mathbf{Y}_l}\mathcal{L}\|_F \le G_l,\qquad
\|\mathbf{H}_l\|_2 \le \beta_l.
\end{equation}
Then
\begin{equation}
|\Delta \mathcal{L}|
\le
\sum_{l=1}^{L}G_l\|\mathbf{R}_l\|_F
+\frac{1}{2}\sum_{l=1}^{L}\beta_l\|\mathbf{R}_l\|_F^2.
\end{equation}
Thus, reducing the layer-wise reconstruction residual reduces a local upper bound on the downstream loss perturbation.

We next apply this view to OMP-MoE$^\dagger$. For token $\mathbf{x}_i$, let the skipped perturbation be
\begin{equation}
\delta \mathbf{y}_i=\sum_{j>m_i^*} g_{i,j}F_{e_{i,j}}(\mathbf{x}_i).
\end{equation}
With the calibrated expert energy
\begin{equation}
\alpha_e \approx \mathbb{E}_{\mathbf{x}}\|F_e(\mathbf{x})\|_2^2,
\end{equation}
OMP-MoE$^\dagger$ predicts the contribution energy as
\begin{equation}
\hat{\mathcal{E}}_{i,j}=g_{i,j}^{2}\alpha_{e_{i,j}}.
\end{equation}
Under a weak-correlation approximation among skipped expert outputs, the skipped output energy is bounded by
\begin{equation}
\mathbb{E}\|\delta \mathbf{y}_i\|_2^2
\lesssim
\sum_{j>m_i^*}\hat{\mathcal{E}}_{i,j}.
\end{equation}
Since OMP-MoE$^\dagger$ chooses the smallest $m_i^*$ satisfying
\begin{equation}
\frac{\sum_{j>m_i^*}\hat{\mathcal{E}}_{i,j}}
{\sum_{j=1}^{K}\hat{\mathcal{E}}_{i,j}+\zeta}
\le
\epsilon,
\end{equation}
we obtain
\begin{equation}
\mathbb{E}\|\delta \mathbf{y}_i\|_2^2
\lesssim
\epsilon\sum_{j=1}^{K}\hat{\mathcal{E}}_{i,j}.
\end{equation}
Applying the local sensitivity bound at the token level gives
\begin{equation}
\mathbb{E}|\Delta \mathcal{L}_i|
\lesssim
C_{1,i}\sqrt{\epsilon}+C_{2,i}\epsilon,
\end{equation}
where $C_{1,i}$ and $C_{2,i}$ absorb the local sensitivity constants and the total predicted expert energy. This result explains why the energy coefficient can skip low-energy expert directions while keeping the induced loss perturbation controlled.

\begin{algorithm}[h]
    \caption{Comprehensive OMP-MoE Pruning Pipeline}
    \label{alg:omp_pruning_detailed}
    \begin{algorithmic}[1]
        \Require MoE Model $\mathcal{M}$, Calibration manifold $\mathcal{D}$, Global pruning ratio $P$
        \Ensure Optimal expert subsets $\{E_l\}_{l=1}^L$

        \State \textbf{Phase 1: Statistics Acquisition}
        \For{each layer $l \in \{1, \dots, L\}$}
            \State Execute a singular forward pass over $\mathcal{D}$ to populate the contribution cache
            \State Cache output signal $\mathbf{Y}_l$ and expert contribution atoms $\{\mathbf{V}_{l,e}\}_{e=1}^{N_e}$
            \State Precompute squared Frobenius norms $v_{l,e} = \|\mathbf{V}_{l,e}\|_F^2$ for all $e$
        \EndFor

        \State \textbf{Phase 2: Layer-wise Greedy Expert Selection}
        \For{each layer $l \in \{1, \dots, L\}$}
            \State Initialize residual signal $\mathbf{R}_0 = \mathbf{Y}_l$ and selected set $E_{l,0} = \emptyset$
            \For{iteration $t = 1$ to $N_e$}
                \State Identify $e^* = \arg\max_{e \notin E_{l,t-1}} \left( 2\langle \mathbf{R}_{t-1}, \mathbf{V}_{l,e} \rangle_{\mathcal{H}} - v_{l,e} \right)$
                \State Update residual: $\mathbf{R}_t = \mathbf{R}_{t-1} - \mathbf{V}_{l,e^*}$
                \State Append $e^*$ to importance sequence $\pi_l$ and record residual energy $r_l(t) = \|\mathbf{R}_t\|_F^2$
            \EndFor
        \EndFor

        \State \textbf{Phase 3: Risk-Aware Global Budget Allocation}
        \State Distribute expert quotas $\{k_l\}_{l=1}^L$ by solving the constrained optimization in Eq.~\ref{eq:cross_layer}
        \State Utilize a greedy water-filling strategy guided by marginal gain curves $\{r_l(k)\}$
        \State \textbf{Return} Retained expert indices $\{E_l\}_{l=1}^L$ where $E_l = \{e \in \pi_l \mid \text{rank}(e) \le k_l\}$
    \end{algorithmic}
\end{algorithm}

\section{Extended Algorithmic Procedures and Computational Analysis}
\label{app:algorithmic_details}

In this section, we provide the comprehensive procedural logic for the OMP-MoE framework and conduct a formal analysis of its computational complexity. We delineate the transition from offline statistics collection to online adaptive inference, establishing the theoretical foundations for the efficiency gains reported in our experiments.

\subsection{Comprehensive Pruning and Selection Pipeline}
\label{app:sub:detailed_workflow}

The OMP-MoE pruning pipeline is structured into three distinct phases: signal acquisition, greedy expert pursuit, and global resource allocation. Unlike traditional pruning methods that necessitate iterative model evaluations, our approach isolates the selection logic from the forward pass by operating on cached contribution atoms. The detailed implementation is provided in Algorithm~\ref{alg:omp_pruning_detailed}.

\begin{algorithm}[h]
    \caption{Adaptive Inference Mechanism (OMP-MoE$^\dagger$ )}
    \label{alg:ompplus_inference}
    \begin{algorithmic}[1]
        \Require Token input $\mathbf{x}$, Top-$K$ expert indices $\{e_1, \dots, e_K\}$, Energy coefficients $\{\alpha_e\}$, Residual threshold $\epsilon$
        \Ensure Parsimonious output signal $\mathbf{y}$

        \State Retrieve gating weights $g_j$ for $j \in \{1, \dots, K\}$
        \State Estimate contribution energy: $\hat{\mathcal{E}}_j = g_j^2 \cdot \alpha_{e_j}$
        \State Compute prefix sums of predicted energy: $S_m = \sum_{j=1}^m \hat{\mathcal{E}}_j$
        \State Determine minimal expert cardinality: $m^* = \min \{ m \in [1, K] \mid (1 - \frac{S_m}{S_K + \zeta}) \le \epsilon \}$
        \State Execute FFN sub-networks only for the subset $\{e_1, \dots, e_{m^*}\}$
        \State \textbf{Return} Aggregated output $\mathbf{y} = \sum_{j=1}^{m^*} g_j F_{e_j}(\mathbf{x})$
    \end{algorithmic}
\end{algorithm}
\subsection{Adaptive Online Inference via OMP-MoE$^\dagger$ }
\label{app:sub:ompplus_logic}
To further catalyze inference-time efficiency, we introduce OMP-MoE$^\dagger$ , an adaptive mechanism that dynamically modulates the expert activation count for each input token. By leveraging the precomputed energy coefficients $\alpha_e$, the model predicts the reconstruction sufficiency of the top-$K$ candidates before full execution. This mechanism is formalized in Algorithm~\ref{alg:ompplus_inference}.

\subsection{Computational Complexity and Asymptotic Bounds}
\label{app:sub:complexity_analysis}

We conduct a rigorous evaluation of the computational trajectory associated with OMP MoE by examining the asymptotic scaling across two distinct temporal phases. Let the configuration space be defined by the number of layers $L$, the total expert count per layer $N_e$, the calibration batch size $B$, and the hidden dimensionality $d$.

The Statistics Acquisition phase necessitates a singular forward evaluation of the model over the calibration manifold. The computational overhead of this phase is isomorphic to standard sparse MoE inference: $\mathcal{O}(L \cdot B \cdot K \cdot d_{ffn})$, where $K$ denotes the active expert count and $d_{ffn}$ represents the intermediate dimension of the feed-forward network. The memory footprint required to cache the contribution atoms is $\mathcal{O}(L \cdot N_e \cdot B \cdot d)$. This cost is incurred exclusively once per architecture and remains invariant to the number of pruning ratios under evaluation.

The Recursive Selection complexity describes the iterative greedy pursuit within a non-orthogonal dictionary of expert contribution atoms. For a given layer, the selection of $n$ experts involves $n$ iterations of inner product computations. The total complexity is expressed as:
\begin{equation}
\mathcal{C}_{\text{search}} \approx \sum_{l=1}^{L} \sum_{t=1}^{n} (N_e - t + 1) \cdot Bd = \mathcal{O}(L \cdot n \cdot N_e \cdot Bd).
\end{equation}
This linear scaling with respect to the token batch $B$ ensures that OMP MoE remains computationally tractable even as $N_e$ increases. This property provides a significant advantage over methods requiring combinatorial forward passes. Table~\ref{tab:complexity_comparison_formal} provides a comparative analysis of these complexity bounds against prevailing baselines.

\begin{table}[!htbp]
\centering
\small
\caption{Asymptotic complexity comparison between OMP MoE and state of the art pruning methodologies. We denote $N_{iter}$ as the number of iterations required for combinatorial search and $T_{ft}$ as the duration of the fine-tuning stage.}
\label{tab:complexity_comparison_formal}
\begin{tabular}{lccc}
\toprule
Method & Search Complexity & Forward Passes & Optimization Type \\
\midrule
NAEE & $\mathcal{O}(L \cdot N_{iter} \cdot \binom{N_e}{n})$ & Massive & Combinatorial \\
DiEP & $\mathcal{O}(T_{ft} \cdot L \cdot B \cdot d^2)$ & Extensive & Gradient-based \\
\textbf{OMP-MoE} & $\mathcal{O}(L \cdot n \cdot N_e \cdot Bd)$ & \textbf{Singular} & Greedy Pursuit \\
\bottomrule
\end{tabular}
\end{table}

As summarized in Table~\ref{tab:complexity_comparison_formal}, OMP-MoE replaces the repeated forward evaluations or gradient steps of the baselines with a single statistics-acquisition pass followed by operations on cached atoms. Its search is polynomial in the number of experts and retained experts, whereas the subset enumeration term in NAEE grows combinatorially.

\subsection{Theoretical FLOPs Reduction in OMP-MoE$^\dagger$ }
\label{app:sub:flops_derivation}

We derive the routed-expert FLOPs of OMP-MoE$^\dagger$ to quantify its theoretical efficiency gain. For the gated feed-forward networks used by the evaluated models, one expert applies three linear projections, and the leading per-token expert cost is
\begin{equation}
\mathcal{F}_{\mathrm{expert}}=6d\,d_{\mathrm{ffn}}.
\end{equation}
Let $\mathcal{F}_{\mathrm{fixed}}$ denote the computation that is unchanged by adaptive expert execution, including attention, router evaluation, shared experts, and other dense components. Let $\mathcal{F}_{\mathrm{select}}$ denote the small cost of predicting contribution energy and selecting the prefix length. A standard top-$K$ MoE layer and OMP-MoE$^\dagger$ therefore require
\begin{align}
\mathcal{F}_{\mathrm{baseline}}
&=\mathcal{F}_{\mathrm{fixed}}+K\mathcal{F}_{\mathrm{expert}},\\
\mathcal{F}_{\mathrm{adaptive}}(\epsilon)
&=\mathcal{F}_{\mathrm{fixed}}+\mathcal{F}_{\mathrm{select}}
+\mathbb{E}[m^*\mid\epsilon]\mathcal{F}_{\mathrm{expert}},
\end{align}
where $m^*$ is the minimum number of candidates whose predicted cumulative energy satisfies the threshold $\epsilon$. Because $\mathcal{F}_{\mathrm{select}}$ is negligible relative to an expert FFN, the normalized routed-expert cost is approximately $\mathbb{E}[m^*\mid\epsilon]/K$, and the corresponding efficiency gain is
\begin{equation}
\eta \approx 1-\frac{\mathbb{E}[m^*\mid\epsilon]}{K}.
\end{equation}
Table~\ref{tab:flops_reduction_summary} reports this normalized cost for the measured activation counts at $\epsilon=0.1$.

\begin{table}[!htbp]
\centering
\small
\caption{Normalized routed-expert FLOPs per token under OMP-MoE$^\dagger$ with $\epsilon=0.1$. The baseline cost for each architecture is normalized to 1.000.}
\label{tab:flops_reduction_summary}
\begin{tabular}{lcccc}
\toprule
Architecture & Active Experts ($K$) & $\mathbb{E}[m^* | \epsilon]$ & Normalized FLOPs $\downarrow$ & Efficiency Gain ($\eta$) \\
\midrule
Mixtral-8$\times$7B & 2 & 1.22 & 1.000 $\rightarrow$ 0.610 & 39.0\% \\
DeepSeek-V2-Lite & 6 & 3.90 & 1.000 $\rightarrow$ 0.650 & 35.0\% \\
Qwen3-30B-A3B & 8 & 5.14 & 1.000 $\rightarrow$ 0.643 & 35.8\% \\
\bottomrule
\end{tabular}
\end{table}

Table~\ref{tab:flops_reduction_summary} predicts routed-expert FLOPs reductions from 35.0\% to 39.0\% across the three architectures. The similar range despite different native top-$K$ values indicates that the adaptive criterion scales with the fraction of contribution energy captured, rather than relying on a particular expert count.

\FloatBarrier
\section{Implementation and Configuration}
\label{app:config}

In this section, we provide exhaustive implementation details, model specifications, and hyperparameter configurations to ensure the reproducibility of our results.

\subsection{Model Architecture Specifications}

Table~\ref{tab:model_arch_detailed} provides the detailed architectural parameters for the models evaluated in this work. We also include a detailed breakdown of routing mechanisms in Table~\ref{tab:routing_details}.

We use the exact checkpoints \texttt{Qwen/Qwen3-30B-A3B}, \texttt{deepseek-ai/DeepSeek-V2-Lite}, \texttt{openai/gpt-oss-20b}, and \texttt{mistralai/Mixtral-8x7B-v0.1} in HuggingFace. As formalized in Eq.~\ref{eq:router}, OMP-MoE consumes the weighted expert contributions produced by each checkpoint's architecture-native router. The selection objective is therefore unchanged across softmax-, sigmoid-, or otherwise transformed router scores, provided that the native score transformation, top-$k$ selection, and normalization are retained when constructing $G(\mathbf{X})$.

We do not enable explicit chain-of-thought reasoning in any generative evaluation. For Qwen3, we explicitly set \texttt{enable\_thinking=False}. Calibration and evaluation formatting are distinct: calibration always uses directly tokenized raw text without a chat template, whereas HumanEval uses the \texttt{humaneval\_instruct} task with \texttt{apply\_chat\_template} so that the prompt format matches the evaluated model. This evaluation-time template is not applied to calibration data.

\begin{table}[!htbp]
\centering
\small
\caption{Detailed architectural specifications of the evaluated MoE models.}
\label{tab:model_arch_detailed}
\resizebox{\columnwidth}{!}{
\begin{tabular}{lccccccccc}
\toprule
Model & Total Params & Active Params & Layers ($L$) & Experts ($N_e$) & Active ($K$) & Hidden Size ($d$) & Expert Intermediate & Dense Intermediate \\
\midrule
\texttt{Qwen/Qwen3-30B-A3B} & 30.5B & 3.3B & 48 & 128 & 8 & 2048 & 768 & 6144 \\
\texttt{deepseek-ai/DeepSeek-V2-Lite} & 15.7B & 2.4B & 27 & 64 & 6 & 2048 & 1408 & 10944 \\
\texttt{openai/gpt-oss-20b} & 21.0B & 3.6B & 24 & 32 & 4 & 2880 & 2880 & -- \\
\texttt{mistralai/Mixtral-8x7B-v0.1} & 46.7B & 12.9B & 32 & 8 & 2 & 4096 & 14336 & -- \\
\bottomrule
\end{tabular}
}
\end{table}

The values in Table~\ref{tab:model_arch_detailed} are read directly from each checkpoint's \texttt{config.json}: \texttt{hidden\_size} gives the hidden size, \texttt{moe\_intermediate\_size} (or the expert \texttt{intermediate\_size} field for Mixtral) gives the routed-expert FFN width, and a distinct dense/shared FFN width is reported separately when the configuration provides one. A dash indicates that the checkpoint has no separate dense intermediate-size field.

\begin{table}[!htbp]
\centering
\small
\caption{Detailed routing mechanisms and gating configurations.}
\label{tab:routing_details}
\begin{tabular}{lcc}
\toprule
Model & Top-K & Shared Experts \\
\midrule
Qwen3-30B-A3B & 8 & No \\
DeepSeek-V2-Lite & 6 & Yes \\
GPT-OSS-20B & 4 & No \\
Mixtral-8$\times$7B-v0.1 & 2 & No \\
\bottomrule
\end{tabular}
\end{table}

Together, Tables~\ref{tab:model_arch_detailed} and~\ref{tab:routing_details} cover expert dictionaries ranging from 8 to 128 experts and native activation counts from 2 to 8. DeepSeek-V2-Lite additionally contains shared experts, which remain unchanged during pruning. This architectural diversity supports evaluating OMP-MoE on weighted contribution atoms instead of assuming a single router implementation.

\subsection{Dataset Processing and Calibration}

We use C4 and WikiText-2 for general-domain calibration and perplexity evaluation. For calibration, we extract raw text and tokenize it directly without applying a chat template. We concatenate the resulting token sequences into one token stream, divide that stream into non-overlapping fixed-length blocks of 4,096 tokens, retain the first 64 complete blocks, and discard any incomplete remainder. Thus, each calibration set contains exactly 262,144 tokens. These blocks are used to cache expert contribution atoms and total layer outputs for the OMP statistics-collection phase. The same preprocessing protocol is used for both corpora.

For the domain-specific generative experiments in Appendix~\ref{app:domain_specific_calibration}, we additionally use \texttt{allenai/tulu-3-sft-personas-math} for GSM8K and \texttt{theblackcat102/evol-codealpaca-v1} for HumanEval. For each domain-specific calibration set, we shuffle the training split with seed 42 and construct 64 sequences of 4,096 tokens, again totaling 262,144 tokens. We use raw-token inputs without chat templates during calibration, and OMP-MoE, REAP, and EASY-EP receive identical calibration sequences.

Table~\ref{tab:dataset_processing} details the characteristics and processing configurations of these datasets. We also provide the breakdown of zero-shot evaluation metrics in Table~\ref{tab:benchmark_metrics}.

\begin{table}[!htbp]
\centering
\small
\caption{Characteristics and processing configurations of the calibration datasets. C4 and WikiText-2 are also used for perplexity evaluation.}
\label{tab:dataset_processing}
\resizebox{\columnwidth}{!}{
\begin{tabular}{lcccc}
\toprule
Dataset & Domain & Processing Method & Block Length & Blocks \\
\midrule
C4 & Web Corpus & \makecell{Non-overlapping fixed-length\\chunking} & 4096 & 64 \\
WikiText-2 & Wikipedia & \makecell{Non-overlapping fixed-length\\chunking} & 4096 & 64 \\
Tulu-3 SFT Personas Math & Mathematics & \makecell{Raw-token fixed-length\\construction} & 4096 & 64 \\
Evol-CodeAlpaca-v1 & Code & \makecell{Raw-token fixed-length\\construction} & 4096 & 64 \\
\bottomrule
\end{tabular}
}
\end{table}

Together, Tables~\ref{tab:dataset_processing} and~\ref{tab:benchmark_metrics} distinguish the fixed calibration-token budget from downstream evaluation. Both calibration corpora use the same non-overlapping 64-block construction, while the evaluation tasks retain their task-specific metrics and shot settings.

\begin{table}[!htbp]
\small
\centering
\caption{Benchmark evaluation tasks and their corresponding performance metrics.}
\label{tab:benchmark_metrics}
\resizebox{\columnwidth}{!}{
\begin{tabular}{lccc}
\toprule
Category & Task & Evaluation & Primary Metric \\
\midrule
Perplexity & WikiText-2, C4 & 0-shot & Perplexity $\downarrow$ \\
8 General Tasks & \makecell[c]{ARC-Challenge, ARC-Easy, BoolQ, HellaSwag,\\ MMLU, OpenBookQA, RTE, WinoGrande} & 0-shot & Accuracy $\uparrow$ \\
Compositional Reasoning & BBH & 0-shot & Exact Match $\uparrow$ \\
Mathematics & GSM8K & 5-shot & Exact Match $\uparrow$ \\
Long-context & LongBench & 0-shot & Score $\uparrow$ \\
Code Generation & HumanEval & 0-shot & pass@1 $\uparrow$ \\
\bottomrule
\end{tabular}
}
\end{table}

\subsection{Benchmark Evaluation Details}
\label{app:benchmark}
We evaluate the four MoE models across a total of twelve diverse benchmarks. These models include Qwen3-30B-A3B~\cite{yang2025qwen3}, DeepSeek-V2-Lite~\cite{liu2024deepseekv2}, GPT-OSS-20B~\cite{agarwal2025gptoss}, and Mixtral-8$\times$7B~\cite{jiang2024mixtral}. We report zero-shot accuracy on eight general benchmarks: ARC-Challenge and ARC-Easy~\cite{clark2018arc}, BoolQ~\cite{clark2019boolq}, HellaSwag~\cite{zellers2019hellaswag}, MMLU~\cite{hendrycks2021mmlu}, OpenBookQA~\cite{mihaylov2018openbookqa}, RTE~\cite{Bentivogli2009rte}, and WinoGrande~\cite{sakaguchi2020winogrande}.

We also conduct evaluations on specialized reasoning, long-context, and code-generation tasks. Table~\ref{tab:benchmark_metrics} summarizes the evaluation settings and primary metrics for these benchmarks. The Big Bench Hard (BBH)~\cite{suzgun2022bbh} dataset tests compositional reasoning capabilities through zero-shot evaluation. We report the exact match score for the BBH benchmark. The GSM8K~\cite{cobbe2021gsm8k} dataset assesses mathematical reasoning through 5-shot evaluation. We report the exact match score for this mathematical task. The LongBench~\cite{bai2024longbench} benchmark evaluates long-context modeling proficiency through zero-shot evaluation. We report its comprehensive score metric. HumanEval~\cite{chen2021codex} evaluates code-generation correctness in the zero-shot setting, for which we report pass@1.

For the domain-specific Qwen3-30B-A3B experiments, GSM8K is evaluated in the 5-shot setting. HumanEval is evaluated with the \texttt{humaneval\_instruct} task in \texttt{lm-eval}, using the model's chat template, \texttt{enable\_thinking=False}, greedy decoding, \texttt{max\_new\_tokens=1024}, and batch size 8. The chat template is used for HumanEval evaluation only, and calibration remains based on the raw-token protocol described above.

\subsection{Standardization of Router Renormalization}
After experts are removed, we preserve each architecture's native router transformation and renormalization rule rather than imposing a universal softmax. In the notation of Eq.~\ref{eq:router}, pruning changes only the set of available routed experts. The score transformation, top-$k$ selection, normalization, and any architecture-specific bias correction remain those of the original checkpoint. All pruning and merging baselines receive the same architecture-native treatment. Consequently, OMP-MoE operates on the resulting weighted contribution atoms regardless of whether the router begins from softmax, sigmoid, or another native score transformation.

\subsection{Hyperparameter Display}

Table~\ref{tab:hyperparams_detailed} lists the hyperparameters used for the OMP-MoE pruning and OMP-MoE$^\dagger$  inference stages. We specify the search and statistics collection configurations in Table~\ref{tab:search_configs}.

For the combinatorial search baseline NAEE~\cite{lu2024naee}, the search space grows factorially with the number of experts. To ensure computational feasibility on models with large expert counts, we imposed a computational budget cap. Specifically, we limited the search to a maximum of 200 iterations per layer for all models except Mixtral-8$\times$7B, where the full search space remains tractable. For all other baselines, we strictly adhered to the hyperparameter configurations reported in their original papers.

\begin{table}[!htbp]
\centering
\small
\caption{Hyperparameter settings for OMP-MoE pruning and OMP-MoE$^\dagger$  inference.}
\label{tab:hyperparams_detailed}
\begin{tabular}{llc}
\toprule
Subgroup & Parameter & Value \\
\midrule
\multirow{2}{*}{OMP Search} & risk penalty coefficient $\lambda$ & 3.0 \\
& Numerical stability $\zeta$ & 1e-6 \\
\midrule
\multirow{2}{*}{Global Allocation} & Min expert ratio & 0.375 \\
& Max expert ratio & 0.875 \\
\midrule
\multirow{2}{*}{OMP-MoE$^\dagger$ } & Adaptive threshold epsilon & 0.1 \\
& Numerical stability $\zeta$ & 1e-6 \\
\bottomrule
\end{tabular}
\end{table}

\begin{table}[!htbp]
\centering
\small
\caption{Search and statistics collection configurations.}
\label{tab:search_configs}
\begin{tabular}{lc}
\toprule
Parameter & Configuration \\
\midrule
Calibration Dataset & C4 and WikiText-2 \\
Calibration Samples per Layer & 64 samples \\
Search Batch Size & 1 \\
Vectorization & PyTorch torch.matmul \\
Layer Offloading & Enabled \\
\bottomrule
\end{tabular}
\end{table}

\subsection{Baseline Configuration Details}\label{app:baseline_config}

For baselines, MC-SMoE, NAEE, MoNE, and Shapley-MoE use a uniform pruning ratio across MoE layers, while DiEP, HC-SMoE, HEAPr, and Sub-MoE include cross-layer budget allocation. For NAEE, the combinatorial search space becomes infeasible on models with large expert counts. Therefore, we impose a fixed computational budget and cap its search to 200 iterations per layer for all models except Mixtral-8$\times$7B, where the full search space is still tractable. This setting follows the same evaluation budget used for the timing comparison.

For the additional REAP~\cite{lasby2025reap} and EASY-EP~\cite{Dong2025easyep} comparisons, we use the same Qwen3-30B-A3B checkpoint, pruning ratios, evaluation harness, architecture-native router treatment, and calibration sequences as OMP-MoE. In particular, all three methods receive the same 64-sequence calibration budget and no chat template is applied during calibration. The reported values are reproduced under this shared protocol rather than copied from prior papers.

\subsection{Hardware and Software Environment}

Table~\ref{tab:hardware_env} summarizes the hardware and software environments used for all experiments.

\begin{table}[!htbp]
\centering
\small
\caption{Hardware and software environment configurations.}
\label{tab:hardware_env}
\begin{tabular}{llc}
\toprule
Category & Component & Specification \\
\midrule
\multirow{3}{*}{Hardware} & GPU & 8 $\times$ NVIDIA H20 (96GB VRAM) \\
& CPU & Intel(R) Xeon(R) Platinum 8358 @ 2.60GHz \\
& RAM & 1TB DDR4 \\
\midrule
\multirow{4}{*}{Software} & OS & Ubuntu 22.04.3 LTS \\
& Python & 3.10.12 \\
& PyTorch & 2.9.1+cu124 \\
& Transformers & 4.57.3 \\
\bottomrule
\end{tabular}
\end{table}

Table~\ref{tab:hardware_env} fixes the compute and software stack used for both pruning-time and inference-time measurements. Reporting the accelerator count, memory capacity, and library versions is necessary for interpreting the absolute search and latency values in the subsequent experiments.

\FloatBarrier
\section{Extended Experimental Results}
\label{app:extended_results}

This section expands the aggregate results in the main text with task-level comparisons, compatibility and sensitivity studies, inference measurements, stronger compression settings, scaling results, and specialized generative evaluations. Each subsection states the evaluation question and discusses the corresponding tables.

\subsection{Detailed Results for Zero-shot Tasks}
\label{app:detailed_main_reuslts}

We provide the complete zero-shot evaluation results for 8 diverse tasks on Qwen3-30B-A3B, DeepSeek-V2-Lite, GPT-OSS-20B, and Mixtral-8$\times$7B models at pruning ratios of 25\% and 50\%. The benchmarks encompass ARC challenge~\cite{clark2018arc}, ARC easy, BoolQ~\cite{clark2019boolq}, HellaSwag~\cite{zellers2019hellaswag}, MMLU~\cite{hendrycks2021mmlu}, OpenBookQA~\cite{mihaylov2018openbookqa}, RTE~\cite{Bentivogli2009rte}, and WinoGrande~\cite{sakaguchi2020winogrande}. Table~\ref{app:tab:detailed_main_results_1} and Table~\ref{app:tab:detailed_main_results_2} list the average accuracy scores. We additionally report the variance for each performance metric to demonstrate the statistical stability and reliability of the selection results.

\begin{table}[!htbp]
\centering
\small
\setlength{\extrarowheight}{0pt}
\addtolength{\extrarowheight}{\aboverulesep}
\addtolength{\extrarowheight}{\belowrulesep}
\setlength{\aboverulesep}{0pt}
\setlength{\belowrulesep}{0pt}
\caption{Comprehensive zero-shot performance comparison on Qwen3-30B-A3B and DeepSeek-V2-Lite model. We report accuracy across eight benchmarks at 25\% and 50\% expert pruning ratios, corresponding to retained expert counts of Num=96 and Num=64 for Qwen3, with analogous settings for DeepSeek. Bold indicates the best performance among pruning methods.}
\label{app:tab:detailed_main_results_1}
\resizebox{0.97\textwidth}{!}{
\begin{tabular}{llccccccccc}
\toprule
\textbf{Expert} & \textbf{Method} & \textbf{ARC-c} & \textbf{ARC-e} & \textbf{BoolQ} & \textbf{HellaS.} & \textbf{MMLU} & \textbf{OBQA} & \textbf{RTE} & \textbf{WinoG.} & \textbf{Avg.$\uparrow$} \\
\midrule
\multicolumn{11}{c}{{\cellcolor[rgb]{0.831,0.937,0.984}}\textbf{Qwen3-30B-A3B }} \\
Num=128 & Original & 0.528 & 0.792 & 0.887 & 0.596 & 0.778 & 0.346 & 0.823 & 0.703 & 0.682 \\
\midrule
\multirow{7}{*}{Num=96} & MC-SMoE & 0.392 & 0.647 & 0.777 & 0.415 & 0.540 & 0.318 & \textbf{0.823} & 0.588 & 0.562 \\
 & HC-SMoE & 0.458 & 0.760 & 0.865 & 0.515 & 0.669 & 0.410 & 0.769 & \textbf{0.704} & 0.644 \\
 & NAEE & 0.481 & 0.769 & 0.870 & 0.555 & 0.701 & 0.290 & 0.773 & 0.693 & 0.642 \\
 & DiEP & 0.505 & 0.774 & 0.871 & 0.567 & 0.646 & 0.334 & 0.718 & 0.702 & 0.640 \\
 & MoNE & 0.523 & 0.782 & 0.871 & 0.580 & 0.730 & 0.308 & 0.776 & 0.684 & 0.657 \\
 & Shapley-MoE & 0.458 & 0.712 & 0.841 & 0.524 & - & 0.300 & - & 0.650 & - \\
 & {\cellcolor[rgb]{1,0.988,0.714}}OMP (Ours) & {\cellcolor[rgb]{1,0.988,0.714}}\textbf{0.534±0.008} & {\cellcolor[rgb]{1,0.988,0.714}}\textbf{0.803±0.015} & {\cellcolor[rgb]{1,0.988,0.714}}\textbf{0.890±0.005} & {\cellcolor[rgb]{1,0.988,0.714}}\textbf{0.592±0.005} & {\cellcolor[rgb]{1,0.988,0.714}}\textbf{0.747±0.003} & {\cellcolor[rgb]{1,0.988,0.714}}\textbf{0.338±0.021} & {\cellcolor[rgb]{1,0.988,0.714}}0.805±0.024 & {\cellcolor[rgb]{1,0.988,0.714}}0.699±0.013 & {\cellcolor[rgb]{1,0.988,0.714}}\textbf{0.676±0.005} \\
\midrule
\multirow{7}{*}{Num=64} & MC-SMoE & 0.195 & 0.309 & 0.583 & 0.273 & 0.261 & 0.250 & 0.729 & 0.470 & 0.384 \\
 & HC-SMoE & 0.358 & 0.652 & 0.830 & 0.418 & 0.461 & \textbf{0.362} & 0.556 & 0.652 & 0.536 \\
 & NAEE & 0.369 & 0.656 & 0.805 & 0.467 & 0.521 & 0.262 & 0.614 & 0.665 & 0.545 \\
 & DiEP & 0.389 & 0.678 & 0.762 & 0.486 & 0.299 & 0.244 & 0.599 & 0.634 & 0.511 \\
 & MoNE & 0.387 & 0.710 & 0.866 & \textbf{0.553} & 0.516 & 0.286 & 0.729 & 0.673 & 0.590 \\
 & Shapley-MoE & 0.320 & 0.512 & 0.712 & 0.420 & - & 0.222 & - & 0.591 & - \\
 & {\cellcolor[rgb]{1,0.988,0.714}}OMP (Ours) & {\cellcolor[rgb]{1,0.988,0.714}}\textbf{0.533±0.009} & {\cellcolor[rgb]{1,0.988,0.714}}\textbf{0.795±0.015} & {\cellcolor[rgb]{1,0.988,0.714}}\textbf{0.879±0.006} & {\cellcolor[rgb]{1,0.988,0.714}}0.546±0.005 & {\cellcolor[rgb]{1,0.988,0.714}}\textbf{0.605±0.004} & {\cellcolor[rgb]{1,0.988,0.714}}0.324±0.021 & {\cellcolor[rgb]{1,0.988,0.714}}\textbf{0.765±0.026} & {\cellcolor[rgb]{1,0.988,0.714}}\textbf{0.695±0.013} & {\cellcolor[rgb]{1,0.988,0.714}}\textbf{0.643±0.005} \\
\hline\hline
\multicolumn{11}{c}{{\cellcolor[rgb]{0.831,0.937,0.984}}\textbf{DeepSeek-V2-Lite}} \\
Num=64 & Original & 0.465 & 0.784 & 0.799 & 0.587 & 0.551 & 0.348 & 0.625 & 0.710 & 0.609 \\
\midrule
\multirow{7}{*}{Num=48} & MC-SMoE & 0.367 & 0.608 & 0.713 & 0.531 & 0.422 & \textbf{0.366} & 0.585 & 0.687 & 0.535 \\
 & HC-SMoE & 0.420 & 0.732 & 0.722 & 0.560 & 0.458 & 0.280 & 0.549 & 0.695 & 0.552 \\
 & NAEE & 0.375 & 0.722 & 0.669 & 0.531 & 0.365 & 0.290 & 0.549 & 0.669 & 0.521 \\
 & DiEP & 0.433 & 0.747 & 0.740 & 0.550 & \textbf{0.505} & 0.302 & 0.588 & 0.660 & 0.566 \\
 & MoNE & 0.456 & 0.774 & 0.720 & 0.576 & 0.474 & 0.314 & 0.549 & 0.702 & 0.571 \\
 & Shapley-MoE & 0.395 & 0.677 & \textbf{0.760} & 0.512 & - & 0.262 & - & 0.673 & - \\
 & {\cellcolor[rgb]{1,0.988,0.714}}OMP (Ours) & {\cellcolor[rgb]{1,0.988,0.714}}\textbf{0.456±0.009} & {\cellcolor[rgb]{1,0.988,0.714}}\textbf{0.774±0.015} & {\cellcolor[rgb]{1,0.988,0.714}}0.713±0.008 & {\cellcolor[rgb]{1,0.988,0.714}}\textbf{0.583±0.005} & {\cellcolor[rgb]{1,0.988,0.714}}0.494±0.004 & {\cellcolor[rgb]{1,0.988,0.714}}0.320±0.021 & {\cellcolor[rgb]{1,0.988,0.714}}0.574±0.029 & {\cellcolor[rgb]{1,0.988,0.714}}0.697±0.013 & {\cellcolor[rgb]{1,0.988,0.714}}\textbf{0.576±0.005} \\
\midrule
\multirow{7}{*}{Num=32} & MC-SMoE & 0.288 & 0.366 & 0.595 & 0.438 & 0.237 & 0.306 & 0.516 & 0.598 & 0.418 \\
 & HC-SMoE & 0.325 & 0.593 & 0.564 & 0.456 & \textbf{0.302} & 0.216 & 0.563 & 0.625 & 0.456 \\
 & NAEE & 0.328 & 0.620 & 0.583 & 0.409 & 0.265 & 0.218 & 0.567 & 0.585 & 0.447 \\
 & DiEP & 0.294 & 0.589 & \textbf{0.617} & 0.441 & 0.256 & 0.212 & 0.498 & 0.569 & 0.435 \\
 & MoNE & \textbf{0.375} & 0.644 & 0.548 & 0.484 & 0.253 & 0.258 & \textbf{0.570} & 0.611 & 0.468 \\
 & Shapley-MoE & 0.268 & 0.491 & 0.240 & 0.395 & - & 0.240 & - & 0.603 & - \\
 & {\cellcolor[rgb]{1,0.988,0.714}}OMP (Ours) & {\cellcolor[rgb]{1,0.988,0.714}}0.354±0.010 & {\cellcolor[rgb]{1,0.988,0.714}}\textbf{0.679±0.014} & {\cellcolor[rgb]{1,0.988,0.714}}0.602±0.008 & {\cellcolor[rgb]{1,0.988,0.714}}\textbf{0.513±0.005} & {\cellcolor[rgb]{1,0.988,0.714}}0.273±0.004 & {\cellcolor[rgb]{1,0.988,0.714}}\textbf{0.310±0.020} & {\cellcolor[rgb]{1,0.988,0.714}}0.549±0.030 & {\cellcolor[rgb]{1,0.988,0.714}}\textbf{0.631±0.014} & {\cellcolor[rgb]{1,0.988,0.714}}\textbf{0.489±0.005} \\
\bottomrule
\end{tabular}
}
\end{table}

\begin{table}[!htbp]
    \centering
    \small
    \setlength{\extrarowheight}{0pt}
    \addtolength{\extrarowheight}{\aboverulesep}
    \addtolength{\extrarowheight}{\belowrulesep}
    \setlength{\aboverulesep}{0pt}
    \setlength{\belowrulesep}{0pt}
    \caption{Comprehensive zero-shot performance comparison on GPT-OSS-20B and Mixtral-8$\times$7B model. We report accuracy across eight benchmarks at 25\% and 50\% expert pruning ratios, corresponding to retained expert counts of Num=24 and Num=16 for GPT-OSS, with analogous settings for Mixtral. Bold indicates the best performance among pruning methods.}

    \label{app:tab:detailed_main_results_2}
    \resizebox{0.97\textwidth}{!}{
    \begin{tabular}{llccccccccc}
    \toprule
    \textbf{Expert} & \textbf{Method} & \textbf{ARC-c} & \textbf{ARC-e} & \textbf{BoolQ} & \textbf{HellaS.} & \textbf{MMLU} & \textbf{OBQA} & \textbf{RTE} & \textbf{WinoG.} & \textbf{Avg.$\uparrow$} \\
    \midrule
    \multicolumn{11}{c}{{\cellcolor[rgb]{0.831,0.937,0.984}}\textbf{GPT-OSS-20B}} \\
    Num=32 & Original & 0.451 & 0.775 & 0.757 & 0.415 & 0.566 & 0.270 & 0.700 & 0.657 & 0.574 \\
    \midrule
    \multirow{7}{*}{Num=24} & MC-SMoE & 0.399 & 0.739 & 0.725 & 0.399 & 0.488 & 0.258 & 0.617 & \textbf{0.665} & 0.536 \\
     & HC-SMoE & 0.285 & 0.592 & 0.622 & 0.336 & 0.425 & 0.182 & 0.610 & 0.617 & 0.459 \\
     & NAEE & 0.422 & 0.729 & 0.731 & 0.400 & 0.547 & 0.232 & 0.664 & 0.624 & 0.544 \\
     & DiEP & 0.379 & 0.712 & 0.751 & 0.393 & 0.527 & 0.224 & \textbf{0.736} & 0.593 & 0.539 \\
     & MoNE & 0.421 & 0.745 & 0.653 & 0.406 & 0.503 & \textbf{0.268} & 0.625 & 0.654 & 0.534 \\
     & {\cellcolor[rgb]{1,0.988,0.714}}OMP (Ours) & {\cellcolor[rgb]{1,0.988,0.714}}\textbf{0.424±0.009} & {\cellcolor[rgb]{1,0.988,0.714}}\textbf{0.765±0.014} & {\cellcolor[rgb]{1,0.988,0.714}}\textbf{0.762±0.008} & {\cellcolor[rgb]{1,0.988,0.714}}\textbf{0.409±0.005} & {\cellcolor[rgb]{1,0.988,0.714}}\textbf{0.555±0.004} & {\cellcolor[rgb]{1,0.988,0.714}}0.260±0.020 & {\cellcolor[rgb]{1,0.988,0.714}}0.693±0.029 & {\cellcolor[rgb]{1,0.988,0.714}}0.650±0.013 & {\cellcolor[rgb]{1,0.988,0.714}}\textbf{0.565±0.005} \\
    \midrule
    \multirow{7}{*}{Num=16} & MC-SMoE & \textbf{0.337} & 0.680 & 0.716 & 0.365 & 0.376 & 0.236 & 0.621 & 0.635 & 0.496 \\
     & HC-SMoE & 0.208 & 0.440 & 0.421 & 0.310 & 0.295 & 0.174 & 0.599 & 0.555 & 0.375 \\
     & NAEE & 0.284 & 0.568 & 0.659 & 0.359 & 0.367 & 0.178 & 0.599 & 0.580 & 0.449 \\
     & DiEP & 0.277 & 0.560 & 0.490 & 0.336 & 0.315 & 0.194 & 0.585 & 0.549 & 0.413 \\
     & MoNE & 0.265 & 0.583 & 0.476 & 0.369 & 0.240 & 0.222 & 0.527 & 0.627 & 0.414 \\
     & {\cellcolor[rgb]{1,0.988,0.714}}OMP (Ours) & {\cellcolor[rgb]{1,0.988,0.714}}0.333±0.009 & {\cellcolor[rgb]{1,0.988,0.714}}\textbf{0.699±0.014} & {\cellcolor[rgb]{1,0.988,0.714}}\textbf{0.750±0.008} & {\cellcolor[rgb]{1,0.988,0.714}}\textbf{0.376±0.005} & {\cellcolor[rgb]{1,0.988,0.714}}\textbf{0.438±0.004} & {\cellcolor[rgb]{1,0.988,0.714}}\textbf{0.256±0.019} & {\cellcolor[rgb]{1,0.988,0.714}}\textbf{0.625±0.029} & {\cellcolor[rgb]{1,0.988,0.714}}\textbf{0.642±0.013} & {\cellcolor[rgb]{1,0.988,0.714}}\textbf{0.515±0.005} \\
    \hline\hline
    \multicolumn{11}{c}{{\cellcolor[rgb]{0.831,0.937,0.984}}\textbf{Mixtral-8$\times$7B}} \\
    Num=8 & Original & 0.565 & 0.842 & 0.851 & 0.649 & 0.671 & 0.350 & 0.711 & 0.759 & 0.675 \\
    \midrule
    \multirow{7}{*}{Num=6} & MC-SMoE & 0.262 & 0.556 & 0.521 & 0.432 & 0.250 & 0.194 & 0.527 & 0.585 & 0.416 \\
     & HC-SMoE & 0.450 & 0.730 & 0.830 & 0.570 & 0.560 & 0.290 & 0.690 & 0.745 & 0.608 \\
     & NAEE & 0.516 & \textbf{0.819} & \textbf{0.836} & 0.616 & 0.587 & 0.330 & 0.679 & 0.754 & 0.642 \\
     & DiEP & 0.514 & 0.809 & 0.835 & 0.612 & 0.598 & 0.302 & 0.657 & 0.740 & 0.633 \\
     & MoNE & 0.463 & 0.775 & 0.702 & 0.594 & 0.493 & 0.300 & 0.545 & 0.728 & 0.575 \\
     & {\cellcolor[rgb]{1,0.988,0.714}}OMP (Ours) & {\cellcolor[rgb]{1,0.988,0.714}}\textbf{0.535±0.008} & {\cellcolor[rgb]{1,0.988,0.714}}0.809±0.015 & {\cellcolor[rgb]{1,0.988,0.714}}0.831±0.007 & {\cellcolor[rgb]{1,0.988,0.714}}\textbf{0.624±0.005} & {\cellcolor[rgb]{1,0.988,0.714}}\textbf{0.604±0.004} & {\cellcolor[rgb]{1,0.988,0.714}}\textbf{0.330±0.021} & {\cellcolor[rgb]{1,0.988,0.714}}\textbf{0.690±0.028} & {\cellcolor[rgb]{1,0.988,0.714}}\textbf{0.766±0.012} & {\cellcolor[rgb]{1,0.988,0.714}}\textbf{0.649±0.005} \\
    \midrule
    \multirow{7}{*}{Num=4} & MC-SMoE & 0.212 & 0.277 & 0.495 & 0.277 & 0.245 & 0.108 & 0.491 & 0.496 & 0.325 \\
     & HC-SMoE & 0.322 & 0.613 & 0.754 & 0.493 & 0.392 & 0.256 & 0.614 & 0.671 & 0.514 \\
     & NAEE & \textbf{0.489} & 0.782 & 0.814 & 0.577 & 0.473 & 0.290 & 0.614 & 0.729 & 0.596 \\
     & DiEP & 0.473 & 0.768 & 0.812 & \textbf{0.580} & 0.489 & 0.292 & 0.599 & \textbf{0.740} & 0.594 \\
     & MoNE & 0.310 & 0.562 & 0.620 & 0.412 & 0.231 & 0.202 & 0.527 & 0.647 & 0.439 \\
     & {\cellcolor[rgb]{1,0.988,0.714}}OMP (Ours) & {\cellcolor[rgb]{1,0.988,0.714}}0.482±0.009 & {\cellcolor[rgb]{1,0.988,0.714}}\textbf{0.781±0.015} & {\cellcolor[rgb]{1,0.988,0.714}}\textbf{0.830±0.007} & {\cellcolor[rgb]{1,0.988,0.714}}0.577±0.005 & {\cellcolor[rgb]{1,0.988,0.714}}\textbf{0.497±0.004} & {\cellcolor[rgb]{1,0.988,0.714}}\textbf{0.298±0.020} & {\cellcolor[rgb]{1,0.988,0.714}}\textbf{0.661±0.029} & {\cellcolor[rgb]{1,0.988,0.714}}0.728±0.012 & {\cellcolor[rgb]{1,0.988,0.714}}\textbf{0.607±0.004} \\
    \bottomrule
    \end{tabular}
    }
\end{table}

Tables~\ref{app:tab:detailed_main_results_1} and~\ref{app:tab:detailed_main_results_2} show that OMP-MoE obtains the highest average accuracy among the compared pruning methods for every evaluated architecture at both pruning ratios. The advantage is particularly visible at 50\% pruning: OMP-MoE reaches 0.643 on Qwen3, 0.489 on DeepSeek-V2-Lite, 0.515 on GPT-OSS-20B, and 0.607 on Mixtral. These consistent task-level results show that the aggregate improvements in the main text are not driven by a single model family.

\subsection{Comparison with One-Shot Expert Pruning Baselines}
\label{app:recent_pruning_baselines}

We further compare OMP-MoE with REAP~\cite{lasby2025reap} and EASY-EP~\cite{Dong2025easyep}, two recent one-shot expert-pruning methods. All methods are evaluated on Qwen3-30B-A3B under the same checkpoint, pruning ratio, router treatment, evaluation harness, and 64-sequence WikiText-2 calibration protocol.

\begin{table}[!htbp]
\centering
\small
\caption{Comparison with REAP and EASY-EP on Qwen3-30B-A3B. Bold indicates the best result among the three pruning methods at the same pruning ratio.}
\label{app:tab:reap_easyep_general}
\resizebox{0.97\textwidth}{!}{
\begin{tabular}{clccccccccc}
\toprule
Pruned & Method & ARC-c & ARC-e & BoolQ & HellaS. & MMLU & OBQA & RTE & WinoG. & Avg. \\
\midrule
0\% & Original & 0.528 & 0.792 & 0.887 & 0.596 & 0.778 & 0.346 & 0.823 & 0.703 & 0.682 \\
\midrule
\multirow{3}{*}{25\%}
& REAP & \textbf{0.555} & 0.797 & 0.867 & 0.579 & 0.733 & 0.330 & 0.791 & 0.690 & 0.668 \\
& EASY-EP & 0.511 & 0.791 & 0.887 & \textbf{0.593} & 0.734 & \textbf{0.344} & \textbf{0.809} & 0.695 & 0.670 \\
& OMP-MoE & 0.534 & \textbf{0.803} & \textbf{0.890} & 0.592 & \textbf{0.747} & 0.338 & 0.805 & \textbf{0.699} & \textbf{0.676} \\
\midrule
\multirow{3}{*}{50\%}
& REAP & 0.455 & 0.741 & 0.821 & 0.464 & 0.546 & 0.316 & 0.737 & 0.651 & 0.591 \\
& EASY-EP & 0.506 & 0.780 & 0.865 & 0.527 & \textbf{0.613} & 0.320 & 0.737 & 0.672 & 0.627 \\
& OMP-MoE & \textbf{0.533} & \textbf{0.795} & \textbf{0.879} & \textbf{0.546} & 0.605 & \textbf{0.324} & \textbf{0.765} & \textbf{0.695} & \textbf{0.643} \\
\bottomrule
\end{tabular}
}
\end{table}

At 50\% pruning, OMP-MoE exceeds REAP by 0.052 and EASY-EP by 0.016 in average accuracy. Among these three recent pruning methods, OMP-MoE achieves the best score on seven of the eight tasks. The larger margin under stronger pruning indicates that residual-dependent expert selection is more useful when the retained expert budget is limited.

\subsection{Compatibility with Orthogonal Compression Methods}
\label{app:orthogonality_detailed}

We conduct an evaluation of the OMP-MoE integrated with post-training quantization and unstructured weight pruning algorithms. Table~\ref{app:tab:detailed_quantization_results} reports the performance of perplexity and zero-shot accuracy on DeepSeek-V2-Lite model at a 25\% pruning ratio. These results demonstrate that OMP MoE remains functional when combined with GPTQ~\cite{frantar2023gptq} or SparseGPT~\cite{frantar2023sparsegpt} or Wanda~\cite{sun2024wanda}. This modularity allows for multi level compression without catastrophic performance degradation, and is complementary to progressive low-bit quantization~\cite{xu2026bit}.

\begin{table}[!htbp]
\centering
\small
\caption{Performance comparison of OMP-MoE combined with post-training quantization and weight pruning methods on DeepSeek-V2-Lite at a 25\% pruning ratio.}
\label{app:tab:detailed_quantization_results}
\resizebox{0.97\textwidth}{!}{
\begin{tabular}{lccccccccccc}
\toprule
\textbf{Methods} & \textbf{C4} $\downarrow$ & \textbf{WikiText-2}$\downarrow$ & \textbf{ARC-c} & \textbf{ARC-e} & \textbf{BoolQ} & \textbf{HellaS.} & \textbf{MMLU} & \textbf{OBQA} & \textbf{RTE} & \textbf{WinoG.} & \textbf{Avg.} $\uparrow$ \\
\midrule
OMP-MoE 25\% & 10.618 & 13.824 & 0.456 & 0.774 & 0.713 & 0.583 & 0.494 & 0.320 & 0.574 & 0.697 & 0.576 \\
+ GPTQ & 12.438 & 14.236 & 0.404 & 0.739 & 0.743 & 0.405 & 0.508 & 0.254 & 0.621 & 0.657 & 0.541 \\
+ SparseGPT (4:8) & 13.709 & 15.673 & 0.369 & 0.711 & 0.720 & 0.501 & 0.368 & 0.302 & 0.578 & 0.655 & 0.525 \\
+ Wanda & 13.321 & 14.430 & 0.394 & 0.728 & 0.720 & 0.501 & 0.390 & 0.308 & 0.585 & 0.649 & 0.534 \\
\bottomrule
\end{tabular}
}
\end{table}

\subsection{Sensitivity Studies on Calibration Set Size}
\label{app:sensitivity_study}

We analyze the impact of calibration data volume and domain on the pruning performance. Table~\ref{tab:calib_size_sensitivity} shows the results for Qwen3-30B-A3B at a 50\% pruning ratio. The accuracy remains stable as the sample size increases from 16 to 256. We find that 64 samples provide sufficient information for near-optimal expert selection. The performance is consistent across both C4~\cite{raffel2020c4} and WikiText-2~\cite{merity2017wikitext2} calibration sets.

\begin{table}[!htbp]
    \centering
    \small
    \caption{{Impact of calibration sample size and dataset source on Qwen3-30B-A3B performance at a 50\% pruning ratio.}}
    \label{tab:calib_size_sensitivity}
    \resizebox{0.97\textwidth}{!}{
    \begin{tabular}{ccccccccccc}
    \toprule
    \multicolumn{1}{l}{\textbf{Calibration Dataset}} & \multicolumn{1}{l}{\textbf{Number of Samples}} & \textbf{ARC-c} & \textbf{ARC-e} & \textbf{BoolQ} & \textbf{HellaS.} & \textbf{MMLU} & \textbf{OBQA} & \textbf{RTE} & \textbf{WinoG.} & \textbf{Avg.}$\uparrow$ \\
    \midrule
    \multirow{5}{*}{C4} & 16 & 0.428 & 0.708 & 0.867 & 0.575 & 0.544 & 0.312 & 0.758 & 0.702 & 0.612 \\
     & 32 & 0.460 & 0.742 & 0.870 & 0.572 & 0.554 & 0.308 & \textbf{0.791} & 0.692 & 0.624 \\
     & 64 & 0.457 & 0.737 & 0.872 & \textbf{0.579} & 0.574 & 0.332 & 0.765 & \textbf{0.703} & 0.627 \\
     & 128 & 0.482 & 0.777 & 0.874 & 0.575 & 0.578 & 0.320 & 0.736 & 0.695 & \textbf{0.630} \\
     & 256 & \textbf{0.497} & \textbf{0.778} & \textbf{0.880} & 0.575 & \textbf{0.583} & \textbf{0.338} & 0.675 & 0.702 & 0.628 \\
    \midrule
    \multirow{5}{*}{WikiText-2} & 16 & 0.517 & 0.778 & 0.873 & 0.548 & 0.606 & 0.316 & 0.762 & 0.681 & 0.635 \\
     & 32 & 0.526 & 0.781 & \textbf{0.880} & 0.550 & 0.617 & \textbf{0.334} & 0.751 & 0.689 & 0.641 \\
     & 64 & \textbf{0.533} & \textbf{0.795} & 0.879 & 0.546 & 0.605 & 0.324 & 0.765 & \textbf{0.695} & \textbf{0.643} \\
     & 128 & 0.521 & 0.792 & 0.878 & 0.549 & 0.614 & 0.314 & 0.729 & 0.692 & 0.636 \\
     & 256 & 0.486 & 0.753 & 0.844 & \textbf{0.564} & \textbf{0.698} & 0.326 & \textbf{0.769} & 0.662 & 0.638 \\
    \bottomrule
    \end{tabular}
    }
\end{table}

\subsection{Ablation of Adaptive Inference (OMP-MoE$^\dagger$ )}
\label{app:ablation_omp_plus}

We evaluate the trade-off between inference efficiency and model fidelity using the OMP-MoE$^\dagger$  mechanism. Table~\ref{tab:omp_plus_ablation} summarizes the performance of DeepSeek-V2-Lite under different thresholds $\epsilon$. We report the average expert activation count, expert skip ratios, and the perplexity on WikiText-2 and C4 datasets, and the zero-shot average accuracy scores on 8 benchmarks. The results indicate that a threshold $\epsilon$ of 0.1 balances FLOPs and accuracy effectively.

\begin{table}[!htbp]
    \centering
    \small
    \caption{Impact of the OMP-MoE$^\dagger$  threshold $\epsilon$ on the inference efficiency and accuracy of DeepSeek-V2-Lite at a 25\% pruning ratio.}
    \label{tab:omp_plus_ablation}
    \begin{tabular}{cccccc}
    \toprule
    \textbf{Threshold $\epsilon$} & \textbf{Average Experts} & \textbf{Skip Ratio} & \textbf{WikiText-2}$\downarrow$ & \textbf{C4} $\downarrow$ & \textbf{Average Accuracy} $\uparrow$\\
    \midrule
    0 (Static) & 6 & 0.00\% & 13.824 & 10.618 & 0.576 \\
    0.001 & 5.97 & 0.50\% & 13.180 & 10.405 & 0.575 \\
    0.01 & 5.73 & 4.50\% & 13.061 & 10.333 & 0.576 \\
    0.05 & 4.68 & 22.00\% & 13.062 & 10.341 & 0.578 \\
    0.1 (Default) & 3.90 & 35.00\% & 13.576 & 10.522 & 0.581 \\
    0.2 & 2.90 & 51.70\% & 13.580 & 10.785 & 0.570 \\
    0.3 & 2.42 & 59.60\% & 14.205 & 11.170 & 0.554 \\
    \bottomrule
    \end{tabular}
\end{table}

\subsection{Empirical Inference Efficiency across Diverse Architectures}
\label{app:sub:empirical_efficiency}

We present the empirical evaluation of inference performance and efficiency trade offs across four distinct Mixture of Experts architectures~\cite{zhu2026fewer}. Tables~\ref{app:tab:gpt-oss-20b-omp++}, \ref{app:tab:mixtral-omp++}, \ref{app:tab:qwen3-30b-a3b-omp++}, and \ref{app:tab:deepseek-v2-lite-omp++} summarize the average zero-shot accuracy, the total computational cost measured in seconds, and the relative speedup for GPT-OSS-20B, Mixtral-8$\times$7B, Qwen3-30B-A3B, and DeepSeek-V2-Lite respectively. Each evaluation compares the static expert pruning framework against the adaptive OMP-MoE$^\dagger$ mechanism under the global budget allocation strategy. The results demonstrate that the synergy between structured pruning and dynamic expert activation yields substantial reductions in execution latency while preserving the linguistic reasoning proficiency of the model across various parameter scales.

\begin{table}[!htbp]
\centering
\small
\caption{Inference performance and efficiency trade offs for the GPT-OSS-20B model.}
\label{app:tab:gpt-oss-20b-omp++}
\begin{tabular}{l|ccc|cc}
\toprule
Pruning Ratio & OMP-MoE & OMP-MoE$^\dagger$ & Avg. Acc & Cost $\downarrow$ & Speedup $\uparrow$ \\
\midrule
0 & - & - & 0.574 & 3188s & 1.00$\times$ \\
0.25 & \checkmark & - & 0.565 & 2140s & 1.49$\times$ \\
0.25 & \checkmark & \checkmark & 0.562 & 2129s & 1.50$\times$ \\
0.50 & \checkmark & - & 0.515 & 1705s & 1.87$\times$ \\
0.50 & \checkmark & \checkmark & 0.510 & 1596s & 2.00$\times$ \\
\bottomrule
\end{tabular}
\end{table}

\begin{table}[!htbp]
\centering
\small
\caption{Inference performance and efficiency trade offs for the Mixtral-8$\times$7B model.}
\label{app:tab:mixtral-omp++}
\begin{tabular}{l|ccc|cc}
\toprule
Pruning Ratio & OMP-MoE & OMP-MoE$^\dagger$ & Avg. Acc & Cost $\downarrow$ & Speedup $\uparrow$ \\
\midrule
0 & - & - & 0.675 & 2954s & 1.00$\times$ \\
0.25 & \checkmark & - & 0.649 & 2685s & 1.10$\times$ \\
0.25 & \checkmark & \checkmark & 0.644 & 2517s & 1.17$\times$ \\
0.50 & \checkmark & - & 0.607 & 2382s & 1.24$\times$ \\
0.50 & \checkmark & \checkmark & 0.600 & 2315s & 1.28$\times$ \\
\bottomrule
\end{tabular}
\end{table}

\begin{table}[!htbp]
\centering
\small
\caption{Inference performance and efficiency trade offs for the Qwen3-30B-A3B model.}
\label{app:tab:qwen3-30b-a3b-omp++}
\begin{tabular}{l|ccc|cc}
\toprule
Pruning Ratio & OMP-MoE & OMP-MoE$^\dagger$ & Avg. Acc & Cost $\downarrow$ & Speedup $\uparrow$ \\
\midrule
0 & - & - & 0.682 & 9980 & 1.00$\times$ \\
0.25 & \checkmark & - & 0.676 & 8911s & 1.12$\times$ \\
0.25 & \checkmark & \checkmark & 0.674 & 8382s & 1.19$\times$ \\
0.50 & \checkmark & - & 0.643 & 6698s & 1.49$\times$ \\
0.50 & \checkmark & \checkmark & 0.641 & 6447s & 1.55$\times$ \\
\bottomrule
\end{tabular}
\end{table}

\begin{table}[!htbp]
\centering
\small
\caption{Inference performance and efficiency trade offs for the DeepSeek-V2-Lite model.}
\label{app:tab:deepseek-v2-lite-omp++}
\begin{tabular}{l|ccc|cc}
\toprule
Pruning Ratio & OMP-MoE & OMP-MoE$^\dagger$ & Avg. Acc & Cost $\downarrow$ & Speedup $\uparrow$ \\
\midrule
0 & - & - & 0.609 & 2459s & 1.00$\times$ \\
0.25 & \checkmark & - & 0.576 & 2222s & 1.11$\times$ \\
0.25 & \checkmark & \checkmark & 0.581 & 2019s & 1.22$\times$ \\
0.50 & \checkmark & - & 0.489 & 1825s & 1.35$\times$ \\
0.50 & \checkmark & \checkmark & 0.476 & 1690s & 1.46$\times$ \\
\bottomrule
\end{tabular}
\end{table}

Across Tables~\ref{app:tab:gpt-oss-20b-omp++}--\ref{app:tab:deepseek-v2-lite-omp++}, adaptive execution provides additional latency reduction after static expert pruning. At 50\% pruning, the combined speedups range from 1.28$\times$ on Mixtral to 2.00$\times$ on GPT-OSS-20B. The corresponding accuracy changes are small on GPT-OSS-20B, Mixtral, and Qwen3, while DeepSeek-V2-Lite exhibits a larger decrease from 0.489 to 0.476, illustrating that the adaptive threshold retains an architecture-dependent accuracy-efficiency trade-off.

\subsection{Robustness and Scaling Results}\label{app:robustness_scaling}

We test robustness beyond the primary 25\% and 50\% pruning settings from three complementary perspectives: more aggressive expert removal, transfer to a substantially larger MoE model, and comparison with smaller dense alternatives. Table~\ref{tab:stronger_retained_ratio} evaluates Qwen3-30B-A3B when only 40\% or 25\% of its experts are retained. Table~\ref{tab:qwen235b} evaluates scaling to Qwen3-235B-A22B, and Table~\ref{tab:dense_comparison} compares the 50\%-pruned Qwen3-30B-A3B with dense Qwen3 models.

\begin{table}[!htbp]
\centering
\small
\caption{Performance under stronger retained-expert constraints on Qwen3-30B-A3B. ``Retained'' denotes the fraction of experts kept in each MoE layer.}
\label{tab:stronger_retained_ratio}
\resizebox{0.97\textwidth}{!}{
\begin{tabular}{c|l|ccccccccc}
\toprule
Retained & Method & ARC-c & ARC-e & BoolQ & HellaS. & MMLU & OBQA & RTE & WinoG. & Avg. \\
\midrule
100\% & Original & 0.528 & 0.792 & 0.887 & 0.596 & 0.778 & 0.346 & 0.823 & 0.703 & 0.682 \\
\midrule
40\% & NAEE & 0.276 & 0.457 & 0.623 & 0.401 & 0.239 & 0.222 & 0.610 & 0.617 & 0.431 \\
40\% & DiEP & 0.299 & 0.562 & 0.775 & 0.425 & \textbf{0.375} & 0.224 & 0.570 & 0.630 & 0.482 \\
40\% & OMP-MoE & \textbf{0.372} & \textbf{0.636} & \textbf{0.848} & \textbf{0.553} & 0.235 & \textbf{0.284} & \textbf{0.736} & \textbf{0.705} & \textbf{0.546} \\
\midrule
25\% & NAEE & 0.195 & 0.283 & 0.419 & 0.262 & \textbf{0.233} & 0.132 & \textbf{0.588} & 0.507 & 0.327 \\
25\% & DiEP & 0.212 & 0.423 & 0.622 & 0.316 & 0.229 & 0.142 & 0.545 & 0.556 & 0.381 \\
25\% & OMP-MoE & \textbf{0.241} & \textbf{0.453} & \textbf{0.712} & \textbf{0.427} & 0.230 & \textbf{0.192} & 0.538 & \textbf{0.636} & \textbf{0.429} \\
\bottomrule
\end{tabular}}
\end{table}

Under aggressive compression in Table~\ref{tab:stronger_retained_ratio}, OMP-MoE achieves average accuracies of 0.546 and 0.429 when retaining 40\% and 25\% of experts, exceeding DiEP by 0.064 and 0.048 and NAEE by 0.115 and 0.102, respectively. Performance declines as expected when three quarters of the experts are removed, but the relative advantage of residual-conditioned selection remains.

\begin{table}[!htbp]
\centering
\small
\caption{Performance on Qwen3-235B-A22B at a 50\% pruning ratio. We report eight-task accuracy and search time.}
\label{tab:qwen235b}
\resizebox{0.97\textwidth}{!}{
\begin{tabular}{c|l|ccccccccc|c}
\toprule
Num & Method & ARC-c & ARC-e & BoolQ & HellaS. & MMLU & OBQA & RTE & WinoG. & Avg. & Time (s) \\
\midrule
128 & Original & 0.618 & 0.854 & 0.896 & 0.678 & 0.849 & 0.372 & 0.812 & 0.787 & 0.733 & -- \\
64 & NAEE & 0.486 & 0.737 & 0.841 & 0.591 & 0.680 & 0.296 & \textbf{0.747} & 0.740 & 0.640 & 20491 \\
64 & DiEP & 0.513 & 0.789 & 0.854 & 0.611 & 0.698 & 0.344 & 0.711 & 0.757 & 0.660 & 10473 \\
64 & OMP-MoE & \textbf{0.559} & \textbf{0.822} & \textbf{0.898} & \textbf{0.667} & \textbf{0.724} & \textbf{0.358} & 0.682 & \textbf{0.774} & \textbf{0.685} & \textbf{846} \\
\bottomrule
\end{tabular}}
\end{table}

The scaling results in Table~\ref{tab:qwen235b} show that OMP-MoE reaches 0.685 average accuracy on Qwen3-235B-A22B, compared with 0.660 for DiEP and 0.640 for NAEE. Its 846-second search is more than 12$\times$ faster than DiEP and 24$\times$ faster than NAEE, indicating that the cached greedy search remains practical as the model size increases.

\begin{table}[!htbp]
\centering
\small
\caption{Comparison between the pruned Qwen3-30B-A3B MoE model and smaller dense Qwen3 models across eight tasks.}
\label{tab:dense_comparison}
\resizebox{0.97\textwidth}{!}{
\begin{tabular}{l|ccccccccc}
\toprule
Model & ARC-c & ARC-e & BoolQ & HellaS. & MMLU & OBQA & RTE & WinoG. & Avg. \\
\midrule
Qwen3-1.7B dense & 0.399 & 0.725 & 0.777 & 0.461 & 0.524 & 0.278 & 0.708 & 0.615 & 0.561 \\
Qwen3-4B dense & 0.503 & 0.752 & 0.850 & 0.522 & 0.530 & 0.296 & 0.747 & 0.654 & 0.607 \\
Qwen3-30B-A3B, 50\% pruned & \textbf{0.533} & \textbf{0.795} & \textbf{0.879} & \textbf{0.546} & \textbf{0.605} & \textbf{0.324} & \textbf{0.765} & \textbf{0.695} & \textbf{0.643} \\
\bottomrule
\end{tabular}}
\end{table}

Table~\ref{tab:dense_comparison} provides a deployment-oriented reference point. The 50\%-pruned Qwen3-30B-A3B obtains an average accuracy of 0.643, outperforming the dense Qwen3-4B and Qwen3-1.7B models by 0.036 and 0.082, respectively. Thus, structured expert removal retains an accuracy advantage over replacing the MoE with these smaller dense checkpoints.

\subsection{Domain-Specific Calibration on Generative Tasks}
\label{app:domain_specific_calibration}

We next test whether task-relevant calibration better preserves specialized experts. For GSM8K, we calibrate on \texttt{allenai/tulu-3-sft-personas-math}. For HumanEval, we use \texttt{theblackcat102/evol-codealpaca-v1}. Each training split is shuffled with seed 42 before constructing 64 sequences of 4,096 tokens. OMP-MoE, REAP, and EASY-EP use exactly the same calibration sequences and raw-token calibration format. We keep the OMP-MoE hyperparameters fixed from the eight-task development setting and do not retune them on GSM8K or HumanEval.

\begin{table}[!htbp]
\centering
\small
\caption{Domain-specific calibration on Qwen3-30B-A3B. GSM8K uses 5-shot exact match with mathematical calibration, while HumanEval uses 0-shot pass@1 with code calibration and the \texttt{humaneval\_instruct} evaluation protocol. Bold indicates the best pruned result at each pruning ratio.}
\label{tab:domain_specific_generative}
\begin{tabular}{llccc}
\toprule
Task & Method & Full (0\%) & 25\% & 50\% \\
\midrule
\multirow{4}{*}{GSM8K}
& Full model & 0.891 & -- & -- \\
& OMP-MoE & -- & \textbf{0.897} & \textbf{0.898} \\
& REAP & -- & 0.895 & 0.895 \\
& EASY-EP & -- & 0.895 & 0.888 \\
\midrule
\multirow{4}{*}{HumanEval}
& Full model & 0.933 & -- & -- \\
& OMP-MoE & -- & \textbf{0.939} & \textbf{0.896} \\
& REAP & -- & 0.927 & 0.878 \\
& EASY-EP & -- & 0.921 & 0.872 \\
\bottomrule
\end{tabular}
\end{table}

With mathematical calibration, OMP-MoE reaches 0.897 and 0.898 exact match on GSM8K at 25\% and 50\% pruning, respectively, compared with 0.891 for the full model. With code calibration, it reaches 0.939 and 0.896 pass@1 on HumanEval, compared with 0.933 for the full model. OMP-MoE also outperforms REAP and EASY-EP at both pruning ratios on both tasks. These results show that the retained expert set is calibration-distribution dependent: general-domain calibration can preserve substantial reasoning ability, while domain-specific calibration can further recover specialized mathematical and code-generation capabilities.

\subsection{Reconstruction Fidelity and Downstream Accuracy}\label{app:reconstruction_accuracy}

We further study the relation between reconstruction error and downstream accuracy on Qwen3-30B-A3B. We compute the total reconstruction error $E$ as the sum of normalized residual rates across all MoE layers. Table~\ref{tab:reconstruction_accuracy_fit} reports this error together with the corresponding average accuracy across eight tasks at five pruning ratios.
\begin{table}[h]
\centering
\small
\caption{Reconstruction error and average downstream performance on Qwen3-30B-A3B.}
\label{tab:reconstruction_accuracy_fit}
\begin{tabular}{c|cc}
\toprule
Pruning Ratio & Reconstruction Error $E$ & Average Accuracy $A$ \\
\midrule
0\% & 0.000 & 0.682 \\
25\% & 0.193 & 0.676 \\
50\% & 2.633 & 0.643 \\
60\% & 5.388 & 0.546 \\
75\% & 17.047 & 0.429 \\
\bottomrule
\end{tabular}
\end{table}
We define normalized reconstruction fidelity as
\begin{equation}
\phi = 1-\frac{E}{17.047},
\end{equation}
where $E$ is the total reconstruction error and $\phi\in[0,1]$. A least-squares fit over the five operating points gives
\begin{equation}
A \approx 0.4170 + 0.2533\phi,\qquad
R^2 \approx 0.947.
\end{equation}
The positive fitted slope shows that higher reconstruction fidelity is associated with higher downstream average accuracy. This supports the use of reconstruction residual as the optimization target in OMP-MoE, while the sensitivity analysis in Appendix~\ref{app:sensitivity_energy} explains why local deviations may occur when residuals lie in low-sensitivity directions.

\FloatBarrier
\section{Visualization and Analytical Experiments}
\label{app:viz}

We provide a comprehensive visual analysis of the OMP-MoE process. This section includes layer-wise results for four different MoE architectures. These models are Qwen3-30B-A3B, DeepSeek-V2-Lite, GPT-OSS-20B, and Mixtral-8$\times$7B.

\subsection{Reconstruction Gain Heatmaps}
\label{app:intro:heatmaps}
We visualize the dynamic importance of experts through reconstruction gain heatmaps. These heatmaps illustrate the relative contribution of each candidate expert at every selection step. A higher color intensity represents a larger reduction in the residual signal. Figure~\ref{apx:fig:app_heatmaps_1} and \ref{apx:fig:app_heatmaps_2} present the layer-wise heatmap results for all four experimental models. These visualizations reveal how the priority of experts shifts during the iterative search process. The OMP algorithm effectively disentangles the correlations between experts by updating the residual signal.

We visualize the reconstruction gain heatmaps to illustrate the dynamic selection of experts throughout the greedy selection process. Each cell within the grid represents the reconstruction gain $\Delta_e$ for a candidate expert at a specific selection iteration. The horizontal axis tracks the progression of selection steps while the vertical axis denotes individual expert indices within the MoE layer. The varying color intensity characterizes the marginal contribution of an expert to the current residual signal. Bright regions indicate experts with high correlation to the uncaptured information manifold, whereas darker regions signify high redundancy relative to the previously selected subset.

These heatmaps reveal the conditional nature of expert importance because the intensity patterns shift significantly after each iteration step. We observe that experts exhibit mutual dependencies during the iterative pursuit. The selection of a specific expert alters the residual signal energy. And this modification in the residual energy subsequently modifies the reconstruction gain of all remaining expert candidates. The pattern of mutual dependencies in experts confirms that one-shot ranking methodologies are insufficient for identifying optimal expert subsets. Furthermore, the patterns highlight that specific experts consistently dominate the initial reconstruction phase in shallow layers, whereas deeper layers exhibit more distributed importance across the expert dictionary. We provide the full layer wise heatmap results for the four primary experimental architectures in Figures~\ref{apx:fig:app_heatmaps_1} and~\ref{apx:fig:app_heatmaps_2}.

\begin{figure*}[!p]
    \centering

    \begin{subfigure}[b]{\linewidth}
        \centering
        \includegraphics[width=\linewidth]{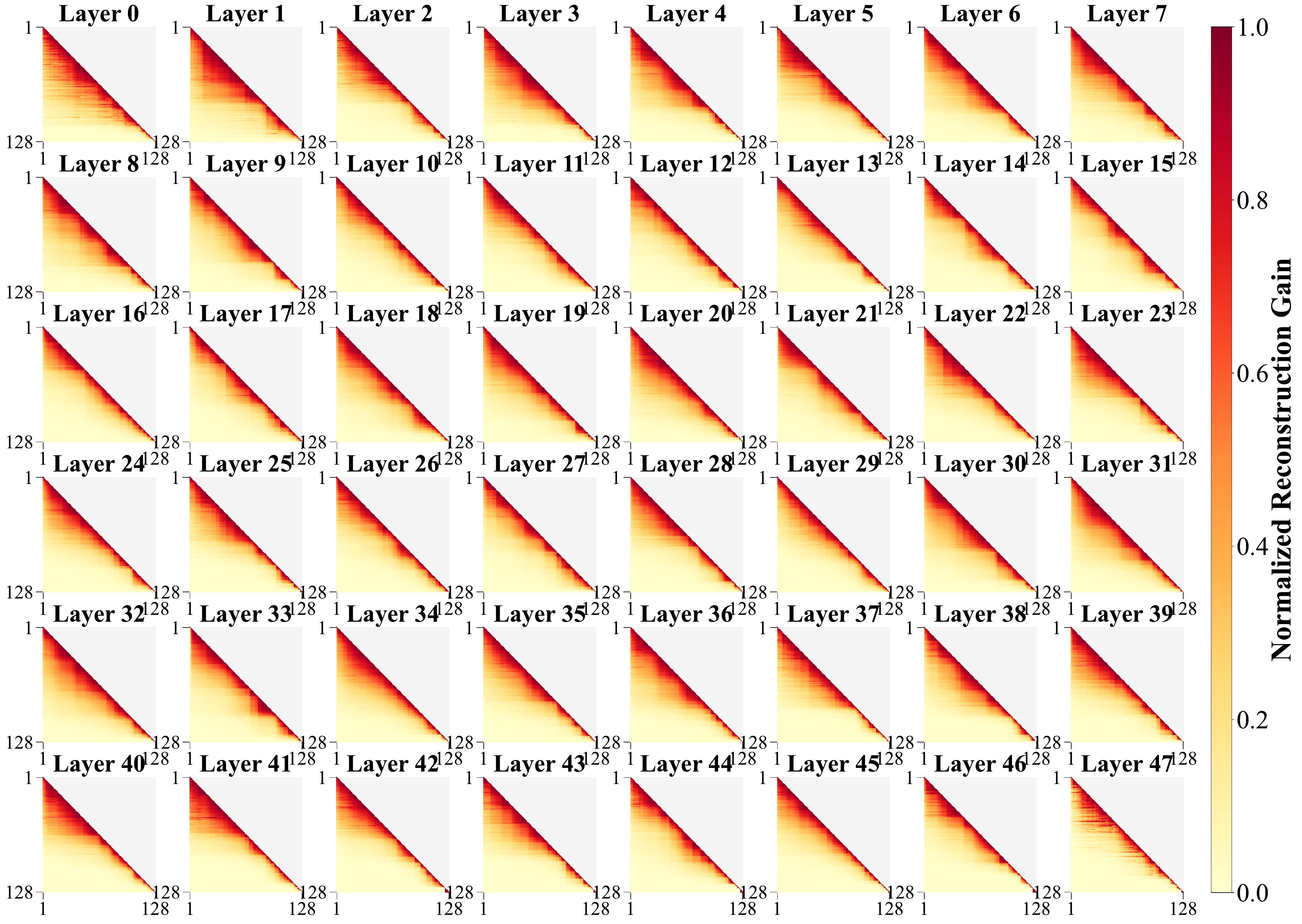}
        \caption{Visualization of Layer-wise reconstruction gain heatmaps within each MoE layer on Qwen3-30B-A3B.}
        \label{apx:fig:Qwen3_heatmap}
    \end{subfigure}

    \begin{subfigure}[b]{\linewidth}
        \centering
        \includegraphics[width=\linewidth]{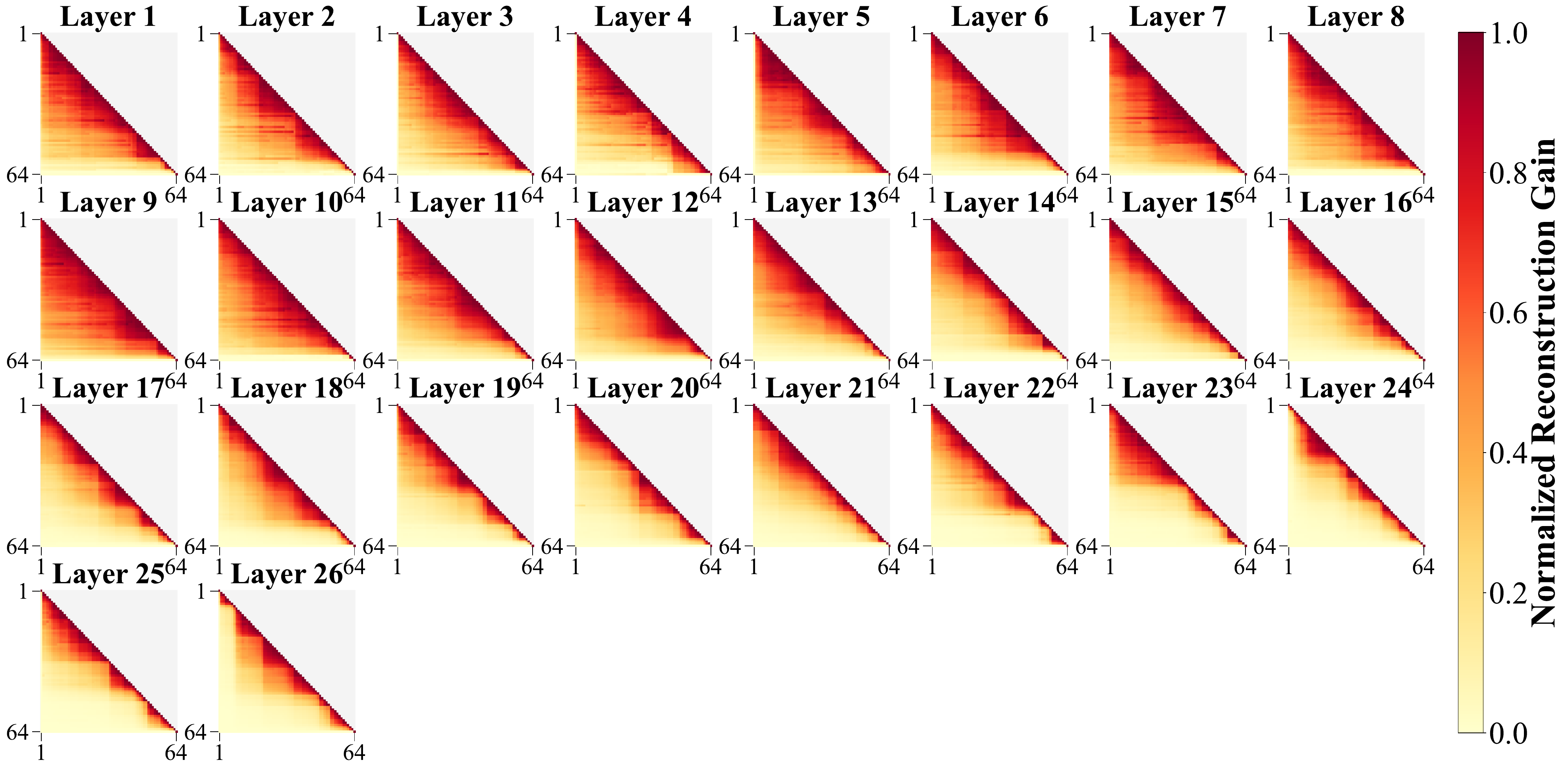}
        \caption{Visualization of Layer-wise reconstruction gain heatmaps within each MoE layer on DeepSeek-V2-Lite.}
        \label{apx:fig:DeepSeek_heatmap}
    \end{subfigure}
    \caption{Layer-wise reconstruction gain heatmaps for the Qwen3-30B-A3B and DeepSeek-V2-Lite models. Each row represents a specific MoE layer and each column represents a selection iteration. Bright cells indicate experts that provide the maximum reconstruction gain at that step.}
    \label{apx:fig:app_heatmaps_1}

\end{figure*}

\begin{figure*}[!p]
    \centering
    \begin{subfigure}[b]{\linewidth}
        \centering
        \includegraphics[width=\linewidth,height=0.36\textheight,keepaspectratio]{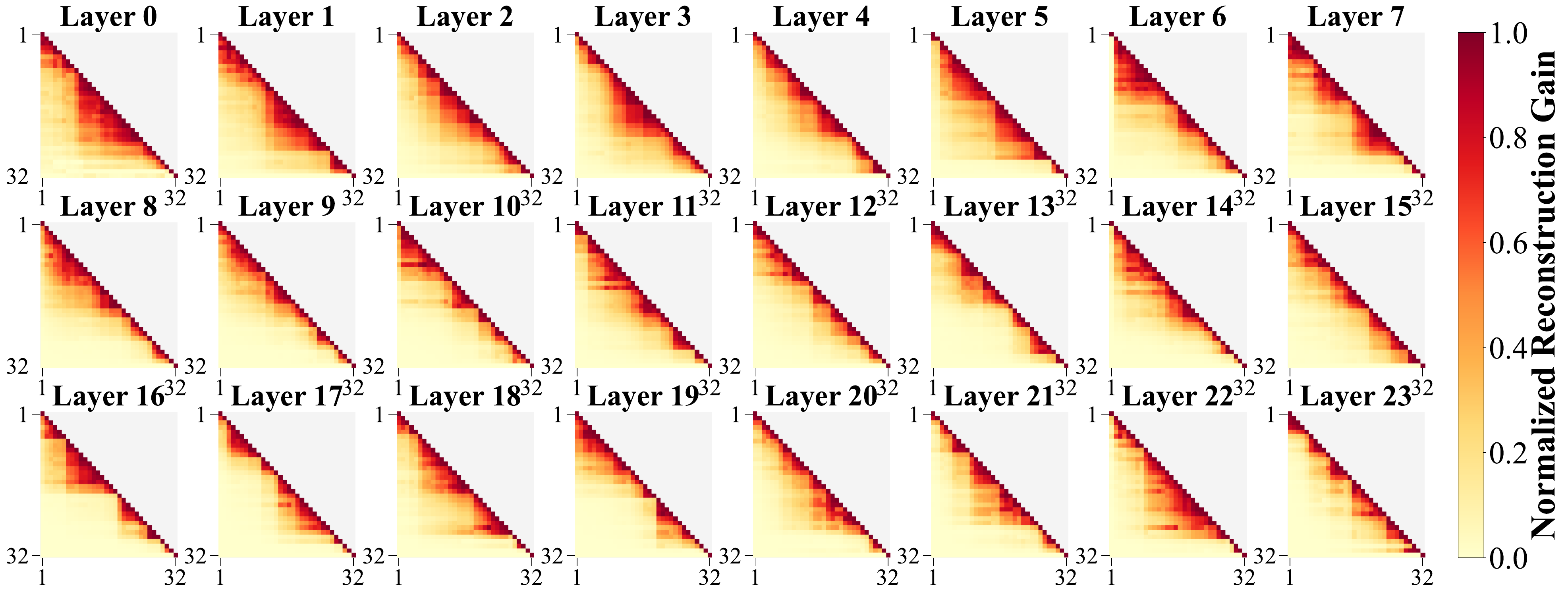}
        \caption{Visualization of Layer-wise reconstruction gain heatmaps within each MoE layer on GPT-OSS-20B.}
        \label{apx:fig:GPT_heatmap}
    \end{subfigure}

    \begin{subfigure}[b]{\linewidth}
        \centering
        \includegraphics[width=\linewidth,height=0.36\textheight,keepaspectratio]{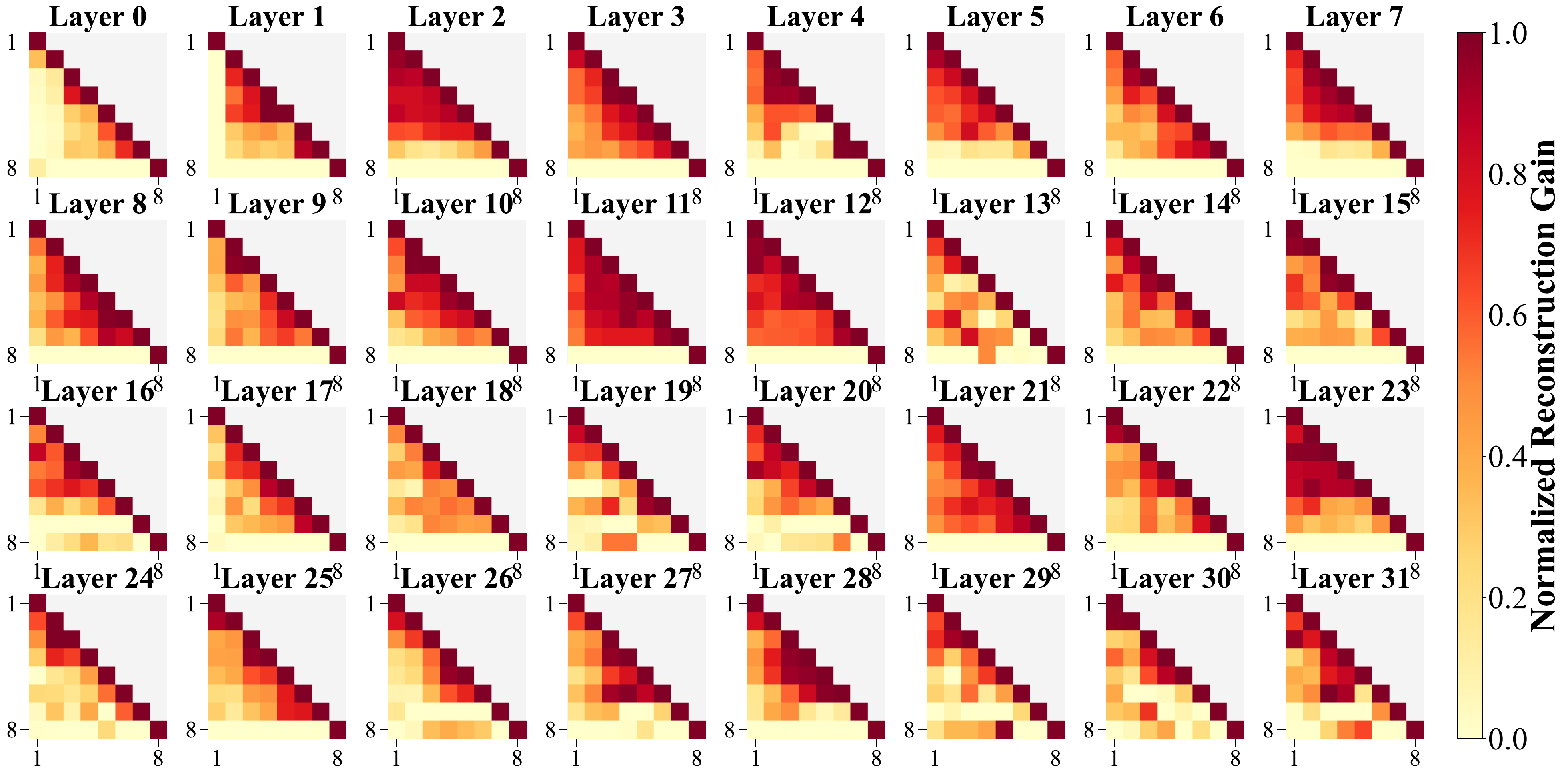}
        \caption{Visualization of Layer-wise reconstruction gain heatmaps within each MoE layer on Mixtral-8$\times$7B.}
        \label{apx:fig:Mixtral_heatmap}
    \end{subfigure}

    \caption{Layer-wise reconstruction gain heatmaps for the GPT-OSS-20B and Mixtral-8$\times$7B models. Each row represents a specific MoE layer and each column represents a selection iteration. Bright cells indicate experts that provide the maximum reconstruction gain at that step.}
    \label{apx:fig:app_heatmaps_2}
\end{figure*}

\FloatBarrier
\subsection{Normalized Residual Curves}
\label{app:intro:residual_curves}
We investigate the convergence properties of the iterative selection process by examining the normalized residual curves across various MoE architectures. The horizontal axis quantifies the count of selected experts while the vertical axis represents the normalized Frobenius norm of the residual signal. Each curve characterizes the proportion of the original signal energy that remains uncaptured after a specific number of selection iterations. We observe that these trajectories typically exhibit a monotonic decay which signifies the progressive reconstruction of the layer output through the Orthogonal Matching Pursuit algorithm.

The gradient of the residual curve provides a direct metric for expert redundancy within a specific MoE layer. A rapid descent indicates that a small subset of experts encapsulates the majority of the information whereas a linear or slow decay suggests that the contribution energy is distributed across a wider dictionary of atoms. Shallow layers frequently demonstrate high signal concentration which allows for aggressive pruning without substantial reconstruction loss. Conversely, deeper layers often require a larger expert budget to satisfy the same reconstruction threshold.

Furthermore, we visualize the risk rate curves to illustrate the stability of the routing distribution as described in Section~\ref{subsec:cross_layer_allocation}. These curves track the preserved routing mass as the expert count increases. The intersection of the residual and risk trajectories highlights the physical trade off between reconstruction fidelity and routing renormalization stability. This visualized heterogeneity across layers justifies our use of a discrete water filling strategy for global budget allocation because it prioritizes layers with the steepest gain to risk ratios. Figures~\ref{apx:fig:app_residual_curves_1} and \ref{apx:fig:app_residual_curves_2} provide the complete layer wise results for all experimental models.

\begin{figure*}[!p]
    \centering
    \begin{subfigure}[b]{\linewidth}
        \centering
        \includegraphics[width=\linewidth,height=0.36\textheight,keepaspectratio]{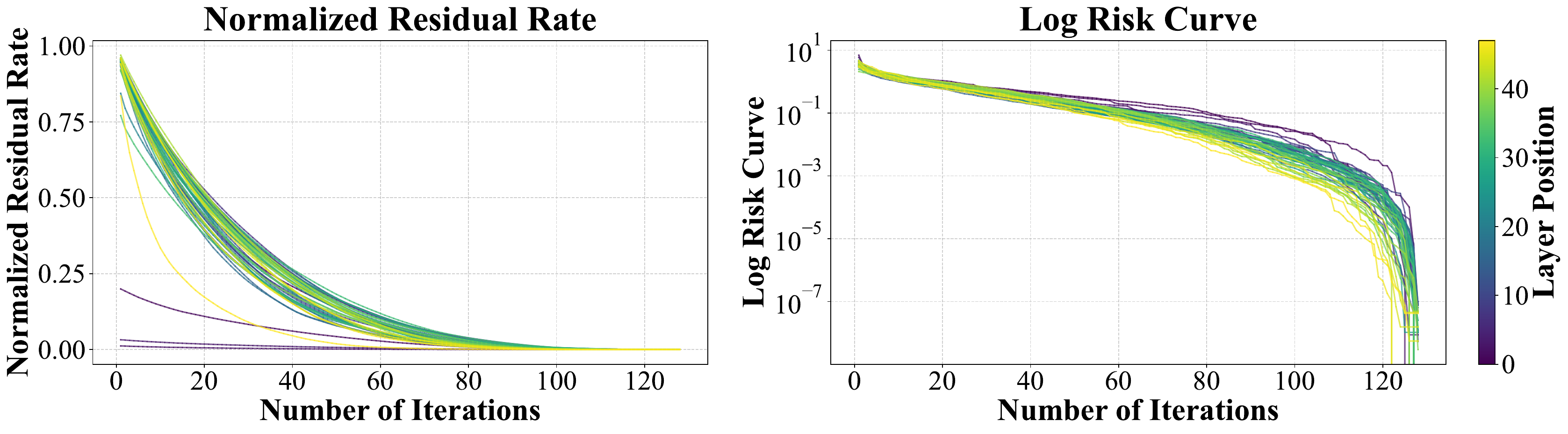}
        \caption{Visualization of normalized residual and risk curves on the Qwen3-30B-A3B model.}
        \label{apx:fig:Qwen3_residual_curves}
    \end{subfigure}

    \begin{subfigure}[b]{\linewidth}
        \centering
        \includegraphics[width=\linewidth,height=0.36\textheight,keepaspectratio]{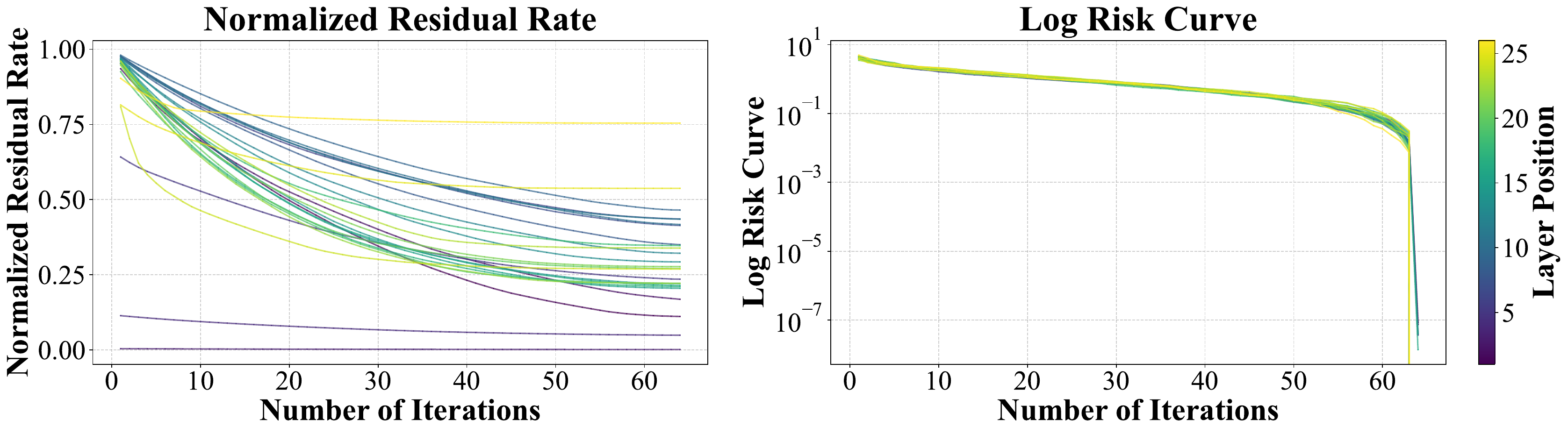}
        \caption{Visualization of normalized residual and risk curves on the DeepSeek-V2-Lite model.}
        \label{apx:fig:DeepSeek_residual_curves}
    \end{subfigure}

    \caption{Normalized residual and risk curves for all layers on Qwen3-30B-A3B and DeepSeek-V2-Lite models. The horizontal axis represents the number of retained experts. The vertical axis represents the normalized residual rate and risk rate. Curves with rapid decay indicate layers with high expert redundancy.}
    \label{apx:fig:app_residual_curves_1}
\end{figure*}
\begin{figure*}[!p]
    \centering
    \begin{subfigure}[b]{\linewidth}
        \centering
        \includegraphics[width=\linewidth,height=0.36\textheight,keepaspectratio]{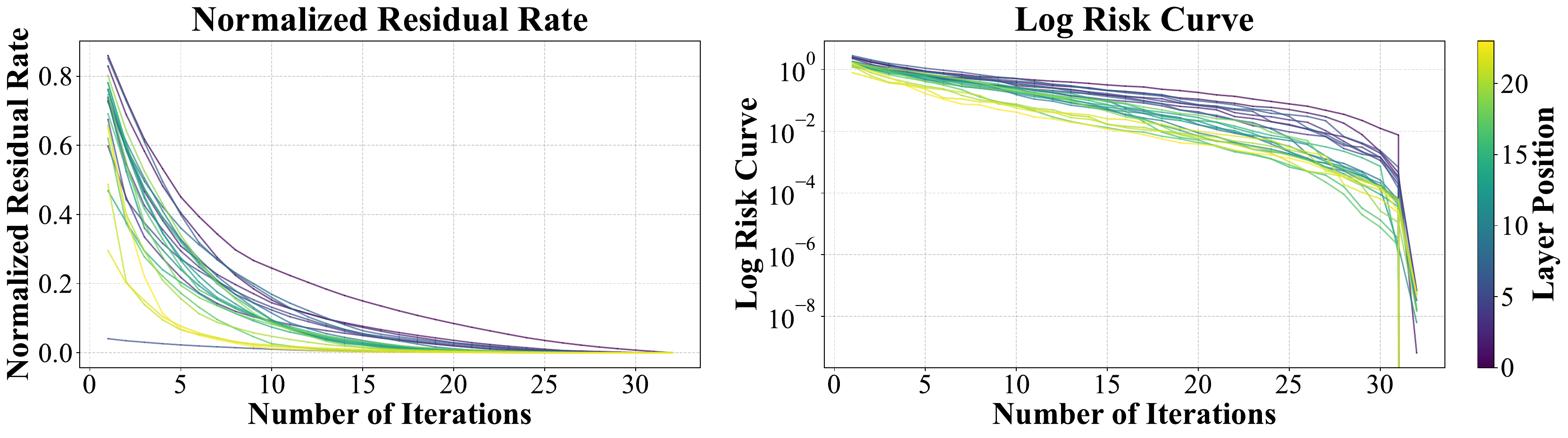}
        \caption{Visualization of normalized residual and risk curves on the GPT-OSS-20B model.}
        \label{apx:fig:GPT_residual_curves}
    \end{subfigure}

    \begin{subfigure}[b]{\linewidth}
        \centering
        \includegraphics[width=\linewidth,height=0.36\textheight,keepaspectratio]{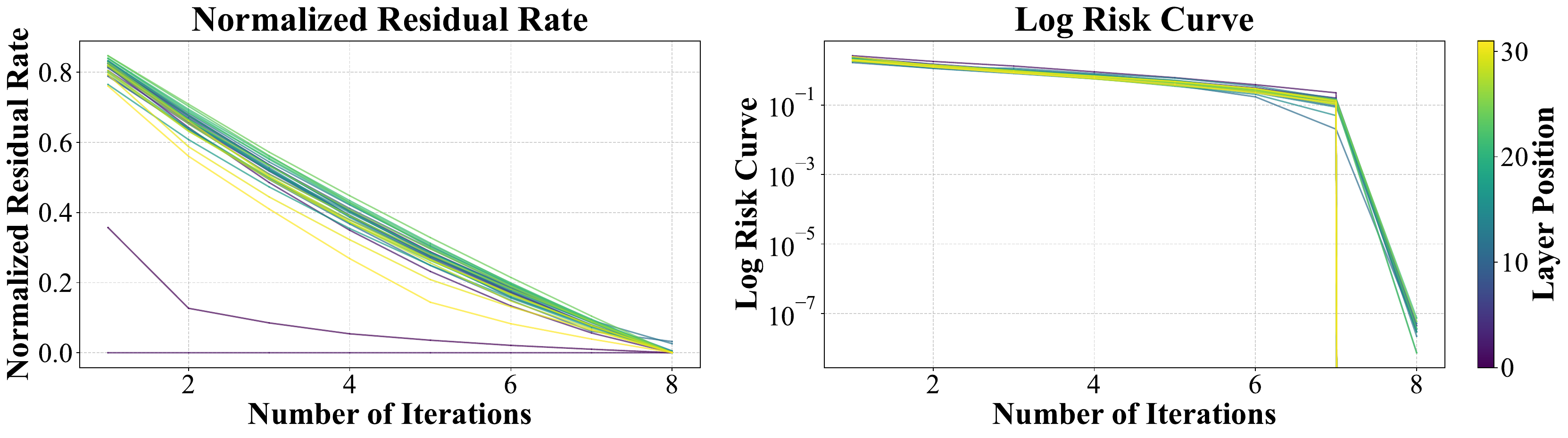}
        \caption{Visualization of normalized residual and risk curves on the Mixtral-8$\times$7B model.}
        \label{apx:fig:Mixtral_residual_curves}
    \end{subfigure}

    \caption{Normalized residual and risk curves for all layers on GPT-OSS-20B and Mixtral-8$\times$7B models. The horizontal axis represents the number of retained experts. The vertical axis represents the normalized residual rate and risk rate. Curves with rapid decay indicate layers with high expert redundancy.}
    \label{apx:fig:app_residual_curves_2}
\end{figure*}

Across Figures~\ref{apx:fig:app_residual_curves_1} and~\ref{apx:fig:app_residual_curves_2}, the decay rates differ substantially across both layers and architectures. This variation supplies the layer-specific marginal-gain curves used by the global allocator rather than assuming uniform redundancy throughout a model.

\FloatBarrier
\subsection{Identification of High-Contribution Experts}
\label{app:super_experts}
We use the descriptive term \textit{high-contribution expert} for an expert whose first-step OMP selection reconstructs more than 90\% of a layer's routed-output energy. Operationally, this occurs when the normalized residual after the first greedy step satisfies $r_l(1)<0.1$, corresponding to a first-step reconstruction gain greater than 0.9. Our definition is specific to the calibration-conditioned OMP objective and quantifies importance through residual reduction relative to the other candidate atoms in the same layer.

This phenomenon is distinct from the \textit{Super Experts} introduced by Su et al.~\cite{su2025superexperts}. That work identifies rare experts through extreme \texttt{down\_proj} activation outliers and studies their data-agnostic role in hidden-state outliers, attention sinks, and downstream performance. In contrast, our analysis neither assigns that established name nor assumes the same activation mechanism: it uses the binary matching-pursuit view to measure how much routed-output energy a candidate uniquely explains under a particular calibration distribution. The two criteria may identify overlapping experts, but they answer different questions and are not treated as equivalent here.

\begin{table}[h]
\centering
\small
\caption{High-contribution experts across four MoE architectures, identified by the initial OMP reconstruction residual.}
\label{apx:tab:super_experts}
\begin{tabular}{lcccc}
\toprule
Model & Layer Index & Residual $r_l(1)$ & Expert ID \\
\midrule
DeepSeek-V2-Lite & 3 & 0.0041 & 54 \\
GPT-OSS-20B & 6 & 0.0407 & 5 \\
Mixtral-8$\times$7B & 1 & 0.0001 & 3 \\
Qwen3-30B-A3B & 2 & 0.0118 & 92 \\
Qwen3-30B-A3B & 3 & 0.0322 & 82 \\
\bottomrule
\end{tabular}
\end{table}

Table~\ref{apx:tab:super_experts} identifies five representative cases across all four architectures. The first-step residual ranges from 0.0001 for Mixtral layer 1 to 0.0407 for GPT-OSS layer 6, so each listed expert reconstructs more than 95\% of its layer's calibration output under this criterion. These cases provide empirical evidence of strong signal concentration: after the dominant atom is selected, the residual-conditioned marginal utility of subsequent experts diminishes rapidly. This OMP-based quantification also motivates the cross-layer allocator, which can assign smaller budgets to layers with rapidly decaying residual curves.

\FloatBarrier
\section{Extended Discussion}
\label{app:extended_discussion}

In this section, we provide clarifications regarding the theoretical boundaries and implementation details of the OMP-MoE framework.

\subsection{Scope of Greedy Optimality}
We explicitly distinguish our selection mechanism from standard Orthogonal Matching Pursuit. The classical algorithm continuously updates coefficients via orthogonal projection. In contrast, expert pruning imposes a binary constraint where coefficients are fixed at one. Consequently, our optimality guarantee is scoped to the local greedy step. OMP-MoE ensures that each iteration selects the expert providing the maximum immediate reduction in residual energy conditioned on the current subset. This formulation avoids matrix inversion overhead while effectively handling mutual dependencies between experts.

\subsection{Superiority over Static Heuristics}
OMP-MoE addresses the inter-expert correlations that static ranking heuristics neglect. Methods based on one-shot routing frequency or norm assume expert independence. However, our heatmaps in Appendix~\ref{app:intro:heatmaps} demonstrate that the utility of a candidate expert shifts dynamically after selecting a dominant expert. By updating the residual signal, OMP-MoE re-evaluates candidates based on uncaptured information. This allows the identification of complementary experts that static magnitude-based methods would discard.


\subsection{Broader Impact}
OMP-MoE aims to reduce the memory and inference cost of MoE LLMs, which may lower the hardware barrier for running large models and reduce energy use during deployment. This can make efficient language technologies more accessible to researchers and users with limited compute resources.

\end{document}